\documentclass[twoside,11pt]{article}

\usepackage{jmlr2e}
\usepackage{mkolar_definitions}
\usepackage{rotating}
\usepackage{multirow}
\usepackage{enumerate}

\usepackage{graphicx}
\usepackage{latexsym}
\usepackage{color}
\usepackage{subfigure}

\let\hat\widehat
\let\tilde\widetilde

\jmlrheading{}{}{}{08/10}{}{Mladen Kolar, John Lafferty and Larry Wasserman}

\ShortHeadings{Union Support Recovery}{Kolar, Lafferty and Wasserman}
\firstpageno{1}

\begin{document}

\title{Union Support Recovery in Multi-task Learning}

\author{\name Mladen Kolar \email mladenk@cs.cmu.edu \\
        \name John Lafferty \email lafferty@cs.cmu.edu \\
        \name Larry Wasserman \email larry@stat.cmu.edu \\
       \addr School of Computer Science\\
       Carnegie Mellon University\\
       5000 Forbes Avenue\\
       Pittsburgh, PA 15213, USA}
\editor{}

\maketitle

\begin{abstract}
  We sharply characterize the performance of different penalization
  schemes for the problem of selecting the relevant variables in the
  multi-task setting.  Previous work focuses on the regression problem
  where conditions on the design matrix complicate the analysis.  A
  clearer and simpler picture emerges by studying the Normal means
  model.  This model, often used in the field of statistics, is a
  simplified model that provides a laboratory for studying complex
  procedures.
\end{abstract}


\section{Introduction}

We consider the problem of estimating a sparse, signal in the presence
of noise.  It has been empirically observed, on various data sets
ranging from cognitive neuroscience \cite{han09blockwise} to
genome-wide association mapping studies \cite{kim09multivariate}, that
considering related estimation tasks jointly, improves estimation
performance.  Because of this, joint estimation from related tasks or
{\it multi-task learning} has received much attention in the machine
learning and statistics community \citep[see for example][and
references therein]{zhang06probabilistic, negahban09Phase,
  obozinski10support, lounici09taking, han09blockwise,
  lounici10oracle, argyriou08convex, kim09multivariate}.  However, the
theory behind multi-task learning is not yet settled.

An example of multi-task learning problem is the problem of estimating
the coefficients of several multiple regression problems
\begin{equation}
  \label{eq:mutliple-regression-model}
  \yb_j = \Xb_j \betab_j + \epsilonb_j, \qquad j \in [k]
\end{equation}
where $\Xb_j \in \RR^{n \times p}$ is the design matrix, $\yb_j \in
\RR^{n}$ is the vector of observations, $\epsilonb_j \in \RR^{n}$ is
the noise vector, $\betab_j \in \RR^p$ is the unknown vector of
regression coefficients for the $j$-th task
and $[n]=\{1, \ldots, n\}$.

When the number of
variables $p$ is much larger than the sample size $n$, it is commonly
assumed that the regression coefficients are jointly sparse, that is,
there exists a small subset $S \subset [p]$, with $s := |S| \ll n$, of
the regression coefficients that are non-zero for all or most of the
tasks. 

The model in \eqref{eq:mutliple-regression-model} under the joint
sparsity assumption was analyzed in, for example,
\cite{obozinski10support}, \cite{lounici09taking},
\cite{negahban09Phase}, \cite{lounici10oracle} and
\cite{kolar10ultrahigh}. \cite{obozinski10support} propose to minimize
the penalized least squares objective with the mixed $(2,1)$-norm of
the coefficients as the penalty term. The authors focus on consistent
estimation of the support set $S$, albeit under the assumption that
the number of tasks $k$ is fixed. \cite{negahban09Phase} use the mixed
$(\infty, 1)$-norm of the coefficients as the penalty term instead and
focus on the exact recovery of the non-zero pattern of the regression
coefficients, rather than the support set $S$. For a rather limited
case of $k=2$, the authors show that when the regression do not share
a common support, it may be harmful to consider the regression
problems jointly using the mixed $(\infty, 1)$-norm
penalty. \cite{kolar10ultrahigh} address the feature selection
properties of the simultaneous greedy forward selection, however, it
is not clear what the benefits are compared to the ordinary forward
selection done on each task separately. In \cite{lounici09taking} and
\cite{lounici10oracle}, the focus is shifted from the consistent
selection to benefits of the joint estimation for the prediction
accuracy and consistent estimation. The number of tasks $k$ is allowed
to increase with the sample size, however, it is assumed that all
tasks share the same features, that is, a relevant coefficient is
non-zero for all tasks.

Despite these previous investigations, the theory is far from settled.
A simple clear picture of when sharing between tasks actually improves
performance has not emerged.  In particular, to the best of our
knowledge, there has been no previous work that sharply characterizes
the performance of different penalization schemes on the problem of
selecting the relevant variables in the multi-task setting.

In this paper we study multi-task learning in the context of the {\em
  many Normal means model}.  This is a simplified model that is often
useful for studying the theoretical properties of procedures.  The use
of the many Normal means model is fairly common in statistics but
appears to be less common in machine learning.

\subsection{The Normal Means Model}

The simplest Normal means model has the form
\begin{equation}\label{eq::mnm}
Y_i = \mu_i + \sigma \epsilon_i,\ \ \ \ \ i=1,\ldots,p
\end{equation}
where
$\mu_1,\ldots,\mu_p$ are unknown parameters
and
$\epsilon_1,\ldots,\epsilon_p$ are
independent, identically distributed Normal
random variables with mean 0 and variance 1.
There are a variety of results
(\cite{brownlow96asymptotic}, \cite{nussbaum96asymptotic})
that show that many learning problems can be converted
into a Normal means problem.
This implies that results obtained in the Normal means setting
can be transferred to many other settings.
As a simple example, consider the nonparametric regression model
$Z_i = m(i/n) + \delta_i$
where $m$ is a smooth function on $[0,1]$ and
$\delta_i \sim N(0,1)$.
Let $\phi_1,\phi_2,\ldots,$ be an 
orthonormal basis on [0,1] and
write
$m(x) = \sum_{j=1}^\infty \mu_j \phi_j(x)$
where $\mu_j = \int_0^1 m(x) \phi_j(x) dx$.
To estimate the regression function $m$ we need only estimate
$\mu_1,\mu_2,\ldots,$.
Let
$Y_j = n^{-1}\sum_{i=1}^n Z_i\, \phi_j(i/n)$.
Then
$Y_j \approx N(\mu_j,\sigma^2)$
where
$\sigma^2 = 1/n$.
This has the form of (\ref{eq::mnm}) with
$\sigma = 1/\sqrt{n}$.
Hence this regression problem can be converted into a Normal means model.

However, the most important aspect of the
Normal means model is that it
allows a clean setting for studying complex problems.
In this paper, we consider the following Normal means model.
Let
\begin{equation}
  \label{eq:normal-means-model}
  Y_{ij} = \left\{
    \begin{array}{cl}
      (1-\epsilon)\Ncal(0,\sigma^2) + \epsilon \Ncal(\mu_{ij}, \sigma^2)
      & \quad j \in [k], \quad i \in S \\
      N(0,\sigma^2)   & \quad j \in [k], \quad i \in S^c \\
    \end{array}\right.
\end{equation}
where $(\mu_{ij})_{i,j}$ are unknown real numbers, 
$\sigma = \sigma_0/ \sqrt{n}$ is the variance with $\sigma_0 > 0$ known,
$(Y_{ij})_{i,j}$ are random observations, $\epsilon \in [0,1]$ is the
parameter that controls the sparsity of features across tasks and $S
\subset [p]$ is the set of relevant features. Let $s = |S|$ denote the
number of relevant features. Denote the matrix $M \in \RR^{p \times
  k}$ of means
\begin{center}
\begin{tabular}{cc|cccc}
& \multicolumn{5}{c}{ Tasks}\\
& & 1 & 2 & $\ldots$ & k \\
\cline{2-6}
& 1 & $\mu_{11}$ & $\mu_{12}$ & $\ldots$ & $\mu_{1k}$ \\
& 2 & $\mu_{21}$ & $\mu_{22}$ & $\ldots$ & $\mu_{2k}$ \\
& $\vdots$ & $\vdots$ & $\vdots$ & $\ddots$ & $\vdots$ \\
& p & $\mu_{p1}$ & $\mu_{p2}$ & $\ldots$ & $\mu_{pk}$ \\
\end{tabular}
\end{center}
and let $\thetab_i = (\mu_{ij})_{j \in [k]}$ denote the $i$-th row of
the matrix $M$. The set $S^c = [p] \bks S$ indexes the zero rows of
the matrix $M$ and the associated observations are distributed
according to the normal distribution with zero mean and variance
$\sigma^2$. The rows indexed by $S$ are non-zero and the corresponding
observation are coming from a mixture of two normal distributions. The
parameter $\epsilon$ determines the proportion of observations coming
from a normal distribution with non-zero mean.  The reader should
regard each column as one vector of parameters that we want to
estimate.  The question is whether sharing across columns improves the
estimation performance.

It is known from the work on the Lasso that in regression problems,
the design matrix needs to satisfy certain conditions in order for the
Lasso to correctly identify the support $S$ \citep[see][for an
extensive discussion on the different
conditions]{geer09conditions}. These regularity conditions are
essentially unavoidable. However, the Normal means
model~\eqref{eq:normal-means-model} allows us to analyze the
estimation procedure in~\eqref{eq:penalized-estimation} and focus on
the scaling of the important parameters $(n, k, p, s, \epsilon,
\mu_{\min})$ for the success of the support recovery. Using the model
\eqref{eq:normal-means-model} and the estimation procedure
in~\eqref{eq:penalized-estimation}, we are able to identify regimes in
which estimating the support is more efficient using the ordinary
Lasso than with the multi-task Lasso and vice versa. Our results
suggest that multi-task Lasso does not outperform the ordinary Lasso
when the features are not considerably shared across tasks and
practitioners should be careful when applying the multi-task Lasso
without knowledge of the task structure.

An alternative representation of the model is
\begin{equation}
  Y_{ij} = \left\{
    \begin{array}{cl}
       \Ncal(\xi_{ij}\mu_{ij}, \sigma^2)
      & \quad j \in [k], \quad i \in S \\
      N(0,\sigma^2)   & \quad j \in [k], \quad i \in S^c \\
    \end{array}\right.
\end{equation}
where $\xi_{ij}$ is a Bernoulli random variable
with success probability $\epsilon$.
Throughout the paper,
we will set $\epsilon = k^{-\beta}$ for some parameter $\beta \in [0,1)$. 
$\beta < 1/2$ corresponds to dense rows and
$\beta > 1/2$ corresponds to sparse rows.
Let $\mu_{\min}$ denote the absolute value of a smallest non-zero
element of $M$, $\mu_{\min} = \min |\mu_{ij}|$.

Under the model~\eqref{eq:normal-means-model}, we analyze the
penalized least squares procedures of the form
\begin{equation}
  \label{eq:penalized-estimation}
  \hat \mub = \argmin_{\mub \in \RR^{p \times k}}\ \frac{1}{2}\norm{\Yb
    - \mub}_F^2 + \pen(\mub)
\end{equation}
where $\norm{A}_F= \sum_{jk}A_{jk}^2$ is the Frobenious norm, $\pen(\cdot)$ is a
penalty function
and $\mub$ is a $p\times k$ matrix of means.
We consider the following penalties
\begin{enumerate}
\item the $\ell_1$ penalty
  \begin{equation*}
    \pen(\mub) = \lambda \sum_{i \in [p]} \sum_{j \in [k]} |\mu_{ij}|,  
  \end{equation*}
  which corresponds to the Lasso procedure applied on each task
  independently, and denote the resulting estimate as $\hat
  \mub^{\ell_1}$
\item the mixed $(2, 1)$-norm penalty
  \begin{equation*}
    \pen(\mub) = \lambda \sum_{i \in [p]} \norm{\theta_i}_2,
  \end{equation*}
  which corresponds to the multi-task Lasso formulation in
  \cite{obozinski10support} and \cite{lounici09taking}, and denote the
  resulting estimate as $\hat \mub^{\ell_1/\ell_2}$
\item the mixed $(\infty, 1)$-norm penalty
  \begin{equation*}
    \pen(\mub) = \lambda \sum_{i \in [p]} \norm{\theta_i}_\infty,    
  \end{equation*}
  which correspond to the multi-task Lasso formulation in
  \cite{negahban09Phase}, and denote the resulting estimate as $\hat
  \mub^{\ell_1/\ell_2}$.
\end{enumerate}
For any solution $\hat \mub$ of \eqref{eq:penalized-estimation}, let
$S(\hat \mub)$ denote the set of estimated non-zero rows
\begin{equation}
  \label{eq:estimated-support}
  S(\hat \mub) = \{ i \in [p]\ :\ \norm{\hat \thetab_i}_2 \neq 0 \}.
\end{equation}
We establish sufficient conditions under which $\PP[S(\hat \mub) \neq S]
\leq \alpha$ for different methods. These results are complemented
with necessary conditions for the recovery of the support set $S$.

\subsection{Overview of the main results}

The main contributions of the paper can be summarized as follows.

\begin{enumerate}
\item We establish a lower bound on the parameter $\mu_{\min}$ as a
  function of the parameters $(n,k,p,s,\beta)$. Our result can be
  interpreted as follows: for any estimation procedure there exists a
  model given by \eqref{eq:normal-means-model} with non-zero elements
  equal to $\mu_{\min}$ such that the estimation procedure will make
  an error when identifying the set $S$ with probability bounded away
  from zero.
\item We establish the sufficient conditions on the signal strength
  $\mu_{\min}$ for the Lasso and both variants of the group Lasso
  under which these procedures can correctly identify the set of
  non-zero rows $S$.
\end{enumerate}

By comparing the lower bounds with the sufficient conditions, we are
able to identify regimes in which each procedure is optimal for the
problem of identifying the set of non-zero rows $S$. Furthermore, we
point out that the usage of the popular group Lasso with the mixed
$(\infty, 1)$ norm can be disastrous when features are not perfectly
shared among tasks. This is further demonstrated using through an
empirical study.

\subsection{Organization of the paper}

The paper is organizes as follows. We start by analyzing the lower
bound for any procedure for the problem of identifying the set of
non-zero rows in $\S$\ref{sec:lower-bound-support}. In
$\S$\ref{sec:upper-bounds-support} we provide sufficient conditions
on the signal strength $\mu_{\min}$ for the Lasso and the group Lasso
to be able to detect the set of non-zero rows $S$. In the following
section, we propose an improved approach to the problem of estimating
the set $S$. Results of a small empirical study are reported in
$\S$\ref{sec:simulation-results}. We close the paper by a discussion
of our findings.

\section{Lower bound on the support recovery}
\label{sec:lower-bound-support}

In this section, we derive a lower bound for the problem of
identifying the correct variables. In particular, we derive conditions
on $(n, k, p, s, \epsilon, \mu_{\min})$ under which any method is
going to make an error when estimating the correct
variables. Intuitively, if $\mu_{\min}$ is very small, a non-zero row
may be hard to distinguish from a zero row. Similarly, if $\epsilon$
is very small, many elements in a row will zero and, again, as a
result it may be difficult to identify a non-zero row. Before, we give
the main result of the section, we introduce the class of models that
are going to be considered.

Let 
\begin{equation*}
  \Fcal[\mu] := \{ \thetab \in \RR^k\ :\ \min_j|\theta_j| \geq
  \mu \}
\end{equation*}
denote the set of feasible non-zero rows. For each $j \in \{0, 1,
\ldots, k\}$, let $\Mcal(j, k)$ be the class of all the subsets of
$\{1, \ldots, k\}$ of cardinality $j$. Let
\begin{equation}
  \label{eq:class-matrices}
  \MM[\mu, s] = \left\{ (\thetab_1, \ldots, \thetab_p)' \in
  \RR^{p \times k}\ :\ \omega \in \Mcal(s, p),
  \ \thetab_i = \left\{ 
    \begin{array}{cc} 
      \in \Fcal[\mu] & \text{ if } i \in \omega\\ 
      \zero & \text{ if } i \not\in \omega 
    \end{array}\right.
  \right\}
\end{equation}
be the class of all feasible matrix means. For a matrix $M \in
\MM[\mu, s]$, let $\PP_M$ denote the joint law of
$\{Y_{ij}\}_{i \in [p], j \in [k]}$. Since $\PP_M$ is a product
measure, we can write $\PP_M = \otimes_{i \in [p]} \PP_{\thetab_i}$.
For a non-zero row $\thetab_i$, we set
\begin{equation*}
  \PP_{\thetab_i}(A) = \int \Ncal(A; \hat \thetab, \sigma^2 \Ib_k)
  d\nu(\hat \thetab), \qquad A \in \Bcal(\RR^k),
\end{equation*}
where $\nu$ is the distribution of the random variable $\sum_{j \in k}
\mu_{ij} \xi_j e_j$ with $\xi_j \sim {\rm Bernoulli}(k^{-\beta})$ and
$\{e_j\}_{j \in [k]}$ denoting the canonical basis of $\RR^k$. For a
zero row $\thetab_i = \zero$, we set 
\begin{equation*}
  \PP_{\zero}(A) = \Ncal(A; \zero, \sigma^2 \Ib_k), \qquad A \in \Bcal(\RR^k).
\end{equation*}
With this notation, we have the following results.

\begin{theorem}
  \label{thm:lower-bound}
  Let
  \begin{equation}
    \label{eq:mu-lower-bound}
    \mu_{\min}^2 = \mu_{\min}^2(n,k,p,s,\epsilon,\beta) = 
     \ln\Big(1 + u + \sqrt{2u + u^2}\Big)\sigma^2 
  \end{equation}
  where
  \begin{equation*}
    u = \frac{\ln \Big( 1 + \frac{\alpha^2 (p-s+1)}{2}\Big)}{2k^{1-2\beta}}.    
  \end{equation*}
  If $\alpha \in (0, \frac{1}{2})$ and $k^{-\beta}u < 1$, then for all
  $\mu \leq \mu_{\min}$,
  \begin{equation}
    \label{eq:minimax-bound}
    \inf_{\hat \mu} \sup_{M \in \MM[\mu, s]} 
       \PP_M[S(\hat \mu) \neq S(M)] \geq \frac{1}{2}(1 - \alpha)
  \end{equation}
  where $\MM[\mu, s]$ is given by \eqref{eq:class-matrices}.
\end{theorem}
The result can be interpreted in words in the following way: whatever
the estimation procedure $\hat \mu$, there exists some matrix $M \in
\MM[\mu_{\min},s]$ such that the probability of incorrectly identifying
the support $S(M)$ is bounded away from zero. In the next section, we
will see that some estimation procedures achieve the lower bound given
in Theorem~\ref{thm:lower-bound}.

\section{Upper bounds on the support recovery}
\label{sec:upper-bounds-support}

In this section, we present sufficient conditions on $(n, p, k,
\epsilon, \mu_{\min})$ for different estimation procedures, so that 
\begin{equation*}
  \PP[S(\hat \mub) \neq S] \leq \alpha.
\end{equation*}
Let $\alpha', \delta' > 0$ be two parameters such that $\alpha' +
\delta' = \alpha$. The parameter $\alpha'$ controls the probability of
making a type one error
\begin{equation*}
  \PP[ \exists i \in [p]\ :\ 
     i \in S(\hat \mub)\text{ and } i \not\in S] \leq \alpha',
\end{equation*}
that is, the parameter $\alpha'$ upper bounds the probability that
there is a zero row of the matrix $M$ that is estimated as a non-zero
row. Likewise, the parameter $\delta'$ controls the probability of
making a type two error
\begin{equation*}
  \PP[ \exists i \in [p]\ :\ 
     i \not\in S(\hat \mub)\text{ and }i \in S] \leq \delta',
\end{equation*}
that is, the parameter $\delta'$ upper bounds the probability that
there is a non-zero row of the matrix $M$ that is estimated as a zero
row. 

The control of the type one and type two errors is established through
the tuning parameter $\lambda$. It can be seen that if the parameter
$\lambda$ is chosen such that, for all $i \in S$, it holds that $\PP[i
\not\in S(\hat \mub)] \leq \delta'/ s$ and, for all $i \in S^c$, it
hold that $\PP[i \in S(\hat \mub)] \leq \alpha' / (p-s)$, then using
the union bound we have that $\PP[S(\hat \mub) \neq S] \leq
\alpha$. In the following subsections, we will use the outlined
strategy to choose $\lambda$ for different estimation procedures.

\subsection{Upper bounds for the Lasso}
\label{sec:upper-bounds-lasso}

Recall that the Lasso estimator is given as
\begin{equation}
  \label{eq:estim-lasso}
  \hat \mub^{\ell_1} = \argmin_{\mub \in \RR^{p \times k}}\ \frac{1}{2}\norm{\Yb
    - \mub}_F^2 + \lambda \norm{\mub}_1.
\end{equation}
It is easy to see that the solution of the above estimation problem is
given as the following soft-thresholding operation
\begin{equation}
  \label{eq:lasso-solution}
  \hat \mu_{ij}^{\ell_1} = \left(1 - \frac{\lambda}{|Y_{ij}|}\right)_+Y_{ij},
\end{equation}
where $(x)_+ := \max(0, x)$. From~\eqref{eq:lasso-solution}, it is
obvious that $i \in S(\hat \mub^{\ell_1})$ if and only if the maximum
statistics, defined as
\begin{equation*}
  M_k(i) = \max_j |Y_{ij}|,
\end{equation*}
satisfies $M_k(i) \geq \lambda$. Therefore it is crucial to find the
critical value of the parameter $\lambda$ such that
\begin{equation*}
  \left\{
    \begin{array}{cccl}
      \mathbb{P}[M_{k}(i) < \lambda]& <& \delta'/s & \quad i
      \in S \\
      \mathbb{P}[M_{k}(i) > \lambda]& <& \alpha'/(p-s) & \quad i
      \in S^c. \\
    \end{array}
  \right.
\end{equation*}
We start by controlling the type one error. For $i \in S^c$ it holds
that 
\begin{equation} \label{eq:lasso-type-one}
  \PP[M_k(i) \geq \lambda] \leq 
    k \PP[|\Ncal(0,\sigma^2)| \geq \lambda] \leq
    \frac{2k\sigma}{\sqrt{2\pi}\lambda}
     \exp\big(-\frac{\lambda^2}{2\sigma^2}\big)
\end{equation}
using lemma~\ref{lem:tail-bound-normal}. Setting the right hand side
to $\alpha'/(p-s)$ in the above display, we obtain that $\lambda$ can
be set as 
\begin{equation}
  \label{eq:lower-bound-lambda-lasso}
  \lambda = \sigma\sqrt{2\ln \frac{2k(p-s)}{\sqrt{2\pi}\alpha'}}
\end{equation}
and \eqref{eq:lasso-type-one} holds as soon as $2\ln
\frac{2k(p-s)}{\sqrt{2\pi}\alpha'} \geq 1$. Next, we deal with the
type two error. Let
\begin{equation}
  \label{eq:binomial-parameter}
  \pi_k = \PP[|(1-\epsilon)\Ncal(0,\sigma^2) + \epsilon
      \Ncal(\mu_{\min},\sigma^2)| > \lambda].
\end{equation}
Then for $i \in S$, $\PP[M_k(i) < \lambda] \leq \PP[{\rm Bin}(k,
\pi_k) = 0]$, where ${\rm Bin}(k, \pi_k)$ denotes the binomial random
variable with parameters $(k, \pi_k)$. Control of the type two error
is going to be established through careful analysis of $\pi_k$ for
various regimes of problem parameters.

\begin{theorem} \label{thm:lasso}
  Let $\lambda$ be defined by \eqref{eq:lower-bound-lambda-lasso}. Suppose
  $\mu_{\min}$ satisfies one of the following two cases:
  \begin{enumerate}[(i)]
  \item $ \mu_{\min} = \sigma\sqrt{2 r \ln k}$ where
    \begin{equation*}
      r > \bigg(\sqrt{1 + C_{k,p,s}} - \sqrt{1 - \beta}\bigg)^2    
    \end{equation*}
    with
    \begin{equation*}
      C_{k,p,s} =\frac{\ln \frac{2(p-s)}{\sqrt{2\pi}\alpha'}}{\ln k}
    \end{equation*}
    and $\lim_{n \rightarrow \infty} C_{k,p,s} \in [0, \infty)$;
  \item $\mu_{\min} \geq \lambda$ when 
    \begin{equation*}
      \lim_{n \rightarrow \infty} \frac{\ln k}{\ln (p-s)} = 0
    \end{equation*}
    and $k^{1-\beta}/2 \geq \ln(s / \delta')$.
  \end{enumerate}
  Then
  \begin{equation*}
    \PP[S(\hat \mub^{\ell_1}) \neq S] \leq \alpha.
  \end{equation*}
\end{theorem}
The proof is given in $\S$\ref{sec:proof-theor-lasso}.

Now we can compare the lower bound on $\mu_{\min}^2$ from Theorem
\ref{thm:lower-bound} and the upper bound from Theorem
\ref{thm:lasso}. Without loss of generality we assume that $\sigma =
1$. We have that when $\beta < 1/2$ the lower bound is of the order
$\Ocal\left(\ln \left( k^{\beta-1/2} \ln(p-s) \right) \right)$ and the
upper bound is of the order $\ln (k(p-s))$. Ignoring the logarithmic
terms in $p$ and $s$, we have that the lower bound is of the order
$\tilde \Ocal(k^{\beta - 1/2})$ and the upper bound is of the order
$\tilde \Ocal(\ln k)$, which implies that the Lasso does not achieve
the lower bound when the non-zero rows are dense. When the non-zero
rows are sparse, $\beta > 1/2$, we have that both the lower and upper
bound are of the order $\tilde \Ocal( \ln k )$ (ignoring the terms
depending on $p$ and $s$).

\subsection{Upper bounds for the group Lasso}
\label{sec:upper-bounds-group}

Recall that the group Lasso estimator is given as
\begin{equation}
  \label{eq:estim-group-lasso}
  \hat \mub^{\ell_1/\ell_2} = \argmin_{\mub \in \RR^{p \times k}}\ \frac{1}{2}\norm{\Yb
    - \mub}_F^2 + \lambda \sum_{i \in [p]} \norm{\thetab_i}_2,
\end{equation}
where $\thetab_i = (\mu_{ij})_{j \in [k]}$. The group Lasso estimator
can be obtained in a closed form as a result of the following
thresholding operation~\citep[see, for example,][]{friedman10note}
\begin{equation}
  \label{eq:group-lasso-solution}
  \hat \thetab_{i}^{\ell_1/\ell_2} = 
  \left(1 - \frac{\lambda}{\norm{Y_{i\cdot}}{2}}
  \right)_+ Y_{i\cdot}
\end{equation}
where
$Y_{i\cdot}$ is the $i^{\rm th}$ row of the data.
From~\eqref{eq:group-lasso-solution}, it is obvious that $i \in S(\hat
\mub^{\ell_1/\ell_2})$ if and only if the statistic defined as
\begin{equation*}
  S_k(i) = \sum_j Y_{ij}^2,
\end{equation*}
satisfies $S_k(i) \geq \lambda$. The choice of $\lambda$ is crucial
for the control of type one and type two errors. We use the following
result, which directly follows from Theorem 2
in~\cite{baraud02non-asymptotic}.
\begin{lemma} \label{lem:baraud-chi-square-test}
  Let $\{Y_i = f_i + \sigma\xi_i\}_{i \in [n]}$ be a sequence of
  independent observations, where $f = \{f_i\}_{i \in [n]}$ is a
  sequence of numbers, $\xi_i \iidsim \Ncal(0,1)$ and $\sigma$
  is a known positive constant. Suppose that $t_{n, \alpha} \in \RR$
  satisfies $\PP[\chi^2_n > t_{n,\alpha}] \leq \alpha$. Let
  \begin{equation*}
    \phi_{\alpha} = I\{ \sum_{i\in[n]} Y_i^2 \geq t_{n,\alpha} \sigma^2 \}
  \end{equation*}
  be a test for $f = 0$ versus $f \neq 0$. Then the test
  $\phi_{\alpha}$ satisfies 
  \begin{equation*}
    \PP[\phi_{\alpha} = 1] \leq \alpha
  \end{equation*}
  when $f = 0$ and 
  \begin{equation*}
    \PP[\phi_{\alpha} = 0] \leq \delta
  \end{equation*}
  for all $f$ such that
  \begin{equation*}
    \norm{f}_2^2 \geq 2(\sqrt{5} + 4)\sigma^2 \ln\left(
      \frac{2e}{\alpha\delta} \right)\sqrt{n}.
  \end{equation*}
\end{lemma}
\begin{proof}
  Follows immediately from Theorem 2 in \cite{baraud02non-asymptotic}.
\end{proof}
It follows directly from lemma~\ref{lem:baraud-chi-square-test} that
setting 
\begin{equation}
  \label{eq:group-lasso:lambda}
  \lambda = t_{n,\alpha'/(p-s)}\sigma^2  
\end{equation}
will control the probability of type one error at the desired level,
that is,
\begin{equation*}
  \PP[S_k(i) \geq \lambda] \leq \alpha'/(p-s), \qquad \forall i \in S^c.
\end{equation*}
The following theorem gives us the control of the type two error.
\begin{theorem} \label{thm:group-lasso}
  Let $\lambda = t_{n,\alpha'/(p-s)}\sigma^2$. Then 
  \begin{equation*}
    \PP[S(\hat \mub^{\ell_1/\ell_2}) \neq S] \leq \alpha
  \end{equation*}
  if 
  \begin{equation*}
    \mu_{\min} \geq \sigma \sqrt{2(\sqrt{5}+4)}
         \sqrt{\frac{k^{-1/2 + \beta}}{1-c}}
         \sqrt{\ln \frac{2e(2s -\delta')(p-s)}{\alpha'\delta'}}
  \end{equation*}
  where $c=\sqrt{2 \ln (2s / \delta') / k^{1-\beta}}$.
\end{theorem}
The proof is given in $\S$\ref{sec:proof-theor-group}.

Using Theorem \ref{thm:lower-bound} and Theorem \ref{thm:group-lasso}
we can compare the lower bound on $\mu_{\min}^2$ and the upper
bound. Without loss of generality we assume that $\sigma = 1$. When
each non-zero row is dense, that is, when $\beta < 1/2$, we have that
both lower and upper bounds are of the order $\tilde \Ocal(k^{\beta -
  1/2})$ (ignoring the logarithmic terms in $p$ and $s$). This suggest
that the group Lasso performs better than the Lasso for the case where
there is a lot of feature sharing between different tasks. Recall from
previous section that the Lasso in this setting does not have the
optimal dependence on $k$. However, when $\beta > 1/2$, that is, in
the sparse non-zero row regime, we see that the lower bound is of the
order $\tilde \Ocal(\ln (k))$ whereas the upper bound is of the order
$\tilde \Ocal( k^{\beta-1/2} )$. This implies that the group Lasso
does not have optimal dependence on $k$ in the sparse non-zero row
setting.

\subsection{Upper bounds for the group Lasso with the mixed $(\infty,
  1)$ norm}
\label{sec:upper-bounds-group-inf}

In this section, we analyze the group Lasso estimator with the mixed
$(\infty, 1)$ norm, defined as
\begin{equation}
  \label{eq:estim-group-lasso-inf}
  \hat \mub^{\ell_1/\ell_\infty} = \argmin_{\mub \in \RR^{p \times k}}\ \frac{1}{2}\norm{\Yb
    - \mub}_F^2 + \lambda \sum_{i \in [p]} \norm{\thetab_i}_{\infty},
\end{equation}
where $\thetab_i = (\mu_{ij})_{j \in [k]}$. The closed form solution
for $\hat \mub^{\ell_1/\ell_\infty}$ can  be obtained
\citep[see][]{han09blockwise}, however, we are only going to use the
following lemma.
\begin{lemma}{\it \citep{han09blockwise}}
  $\hat \thetab_i^{\ell_1/\ell_\infty} = \zero$ if and only if $\sum_j
  |Y_{ij}| \leq \lambda$.
\end{lemma}
\begin{proof}
  See the proof of Proposition 5 in \cite{han09blockwise}.
\end{proof}
Suppose that the penalty parameter $\lambda$ is set as
\begin{equation} \label{eq:lambda-group-inf}
  \lambda = k\sigma\sqrt{2\ln \frac{k(p-s)}{\alpha'}}.
\end{equation}
Then it follows directly from lemma~\ref{lem:tail-bound-normal} that
\begin{equation*}
  \PP[\sum_j|Y_{ij}| \geq \lambda] 
    \leq k \max_j\PP[|Y_{ij}| \geq \lambda/k]
    \leq \alpha'/(p-s), \qquad \forall i \in S^c,
\end{equation*}
which implies that the probability of the type one error is controlled
at the desired level.

\begin{theorem} \label{thm:group-lasso-inf} Let the penalty parameter
  $\lambda$ be defined by~\eqref{eq:lambda-group-inf}. Then
  \begin{equation*}
    \PP[S(\hat \mub^{\ell_1/\ell_\infty}) \neq S] \leq \alpha
  \end{equation*}
  if 
  \begin{equation*}
    \mu_{\min}\ \geq\ \frac{1+\tau}{1-c} k^{-1+\beta}\lambda
  \end{equation*}
  where $c=\sqrt{2 \ln (2s / \delta') / k^{1-\beta}}$ and $\tau =
  \sigma\sqrt{2k\ln \frac{2s - \delta'}{\delta'}} / \lambda$.
\end{theorem}
The proof is given in $\S$\ref{sec:proof-theor-group-inf}.

Comparing upper bounds for the Lasso and the group Lasso with the
mixed $(2,1)$ norm with the result of Theorem
\ref{thm:group-lasso-inf}, we can see that both the Lasso and the
group Lasso have better dependence on $k$ than the group Lasso with
the mixed $(\infty, 1)$ norm. The difference becomes more pronounced
as $\beta$ increases. This suggest that we should be very cautious
when using the group Lasso with the mixed $(\infty, 1)$ norm, since as
soon as the tasks do not share exactly the same features, the other two
procedures have much better performance on identifying the set of
non-zero rows.

\section{Improved estimation procedure}
\label{sec:impr-estim-proc}

We have observed in the last section that the Lasso procedure performs
better than the group Lasso when each non-zero row is sparse, while
the group Lasso (with the mixed $(2,1)$ norm) performs better when
each non-zero row is dense. Since in many practical situations one
does not how much overlap there is between different tasks, it would
be useful to combine the Lasso and the group Lasso in order to improve
the performance. This can be simply done by estimating $S(\hat
\mub^{\ell_1})$ using \eqref{eq:estim-lasso} and $S(\hat
\mub^{\ell_1/\ell_2})$ using \eqref{eq:estim-group-lasso}
separately. Finally, we can combine these estimates by taking their
union $\hat S = S(\hat \mub^{\ell_1}) \cup S(\hat
\mub^{\ell_1/\ell_2})$. The outlined approach has the advantage that
one does not need to know in advance which estimation procedure to
use. From the theoretical analysis of the Lasso and the group Lasso,
we can see that controlling the error of omitting a non-zero row is
more difficult that controlling the probability of falsely including a
zero row. Therefore, combining the Lasso and the group Lasso estimate
can be seen as a way to increase the power to detect the non-zero rows.

\section{Simulation results}
\label{sec:simulation-results}

We conduct a small-scale empirical study of the performance of the
Lasso and the group Lasso (both with the mixed $(2,1)$ norm and with
the mixed $(\infty, 1)$ norm). Our empirical study shows that the
theoretical findings of $\S$\ref{sec:upper-bounds-support} describe
sharply the behavior of procedures even for small sample studies. In
particular, we demonstrate that as the minimum signal level
$\mu_{\min}$ varies in the model \eqref{eq:normal-means-model}, our
theory sharply determines points at which probability of identifying
non-zero rows of matrix $M$ successfully transitions from $0$ to $1$
for different procedures.

The simulation procedure can be described as follows. Without loss of
generality we let $S = [s]$ and draw the samples $\{Y_{ij}\}_{i\in[p],
  j\in[k]}$ according to the model
in~\eqref{eq:normal-means-model}. The total number of rows $p$ is
varied in $\{128, 256, 512, 1024\}$ and the number of columns is set
to $k = \lfloor p\log_2(p) \rfloor$. The sparsity of each non-zero row
is controlled by changing the parameter $\beta$ in $\{0, 0.25, 0.5,
0.75\}$ and setting $\epsilon = k^{-\beta}$. The number of non-zero
rows is set to $s = \lfloor \log_2 (p) \rfloor$, the sample size is
set to $n = 0.1p$ and $\sigma_0 = 1$. The parameters $\alpha'$ and
$\delta'$ are both set to $0.01$. For each setting of the parameters,
we report our results averaged over 1000 simulation runs. Simulations
with other choices of parameters $n, s$ and $k$ have been tried out,
but the results were qualitatively similar and, hence, we do not
report them here.

\subsection{Lasso}

We investigate the performance on the Lasso for the purpose of
estimating the set of non-zero rows,
$S$. Figure~\ref{fig:lasso:exper_4} plots the probability of success
as a function of the signal strength. On the same figure we plot the
probability of success for the group Lasso with both $(2,1)$ and
$(\infty, 1)$-mixed norms. Using theorem \ref{thm:lasso}, we set
\begin{equation}
  \label{eq:exp:lasso}
  \mu_{\rm lasso} = \sqrt{2 (r + 0.001) \ln k}
\end{equation}
where $r$ is defined in theorem \ref{thm:lasso}. Next, we generate
data according to \eqref{eq:normal-means-model} with all elements
$\{\mu_{ij}\}$ set to $\mu = \rho \mu_{\rm lasso}$, where $\rho \in
[0.05,2]$. The penalty parameter $\lambda$ is chosen as in
\eqref{eq:lower-bound-lambda-lasso}. Figure~\ref{fig:lasso:exper_4}
plots probability of success as a function of the parameter $\rho$,
which controls the signal strength. 
This probability transitions very sharply from 0 to 1.
A rectangle on a horizontal line
represents points at which the probability $\PP[\hat S = S]$ is between
$0.05$ and $0.95$. From each subfigure in
Figure~\ref{fig:lasso:exper_4}, we can observe that the probability of
success for the Lasso transitions from $0$ to $1$ for the same value
of the parameter $\rho$ for different values of $p$, which indicated
that, except for constants, our theory correctly characterizes the
scaling of $\mu_{\min}$. In addition, we can see that the Lasso
outperforms the group Lasso (with $(2,1)$-mixed norm) when each
non-zero row is very sparse (the parameter $\beta$ is close to one).

\begin{figure}[!ht]
  \centering
  Probability of successful support recovery: Lasso\\
\subfigure{
    \includegraphics[width=0.54\columnwidth]{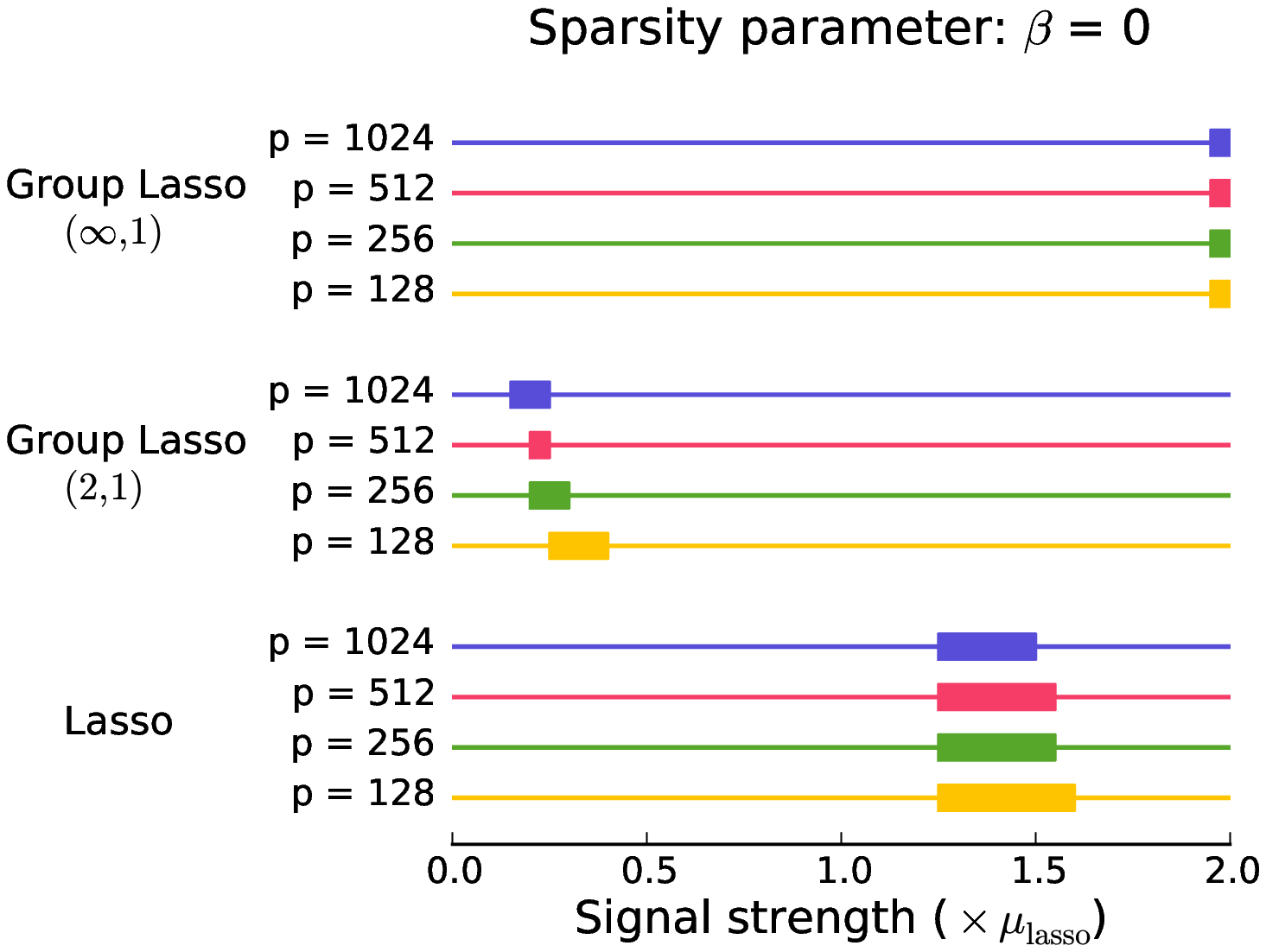}
    } 
\subfigure{
    \includegraphics[width=0.36\columnwidth]{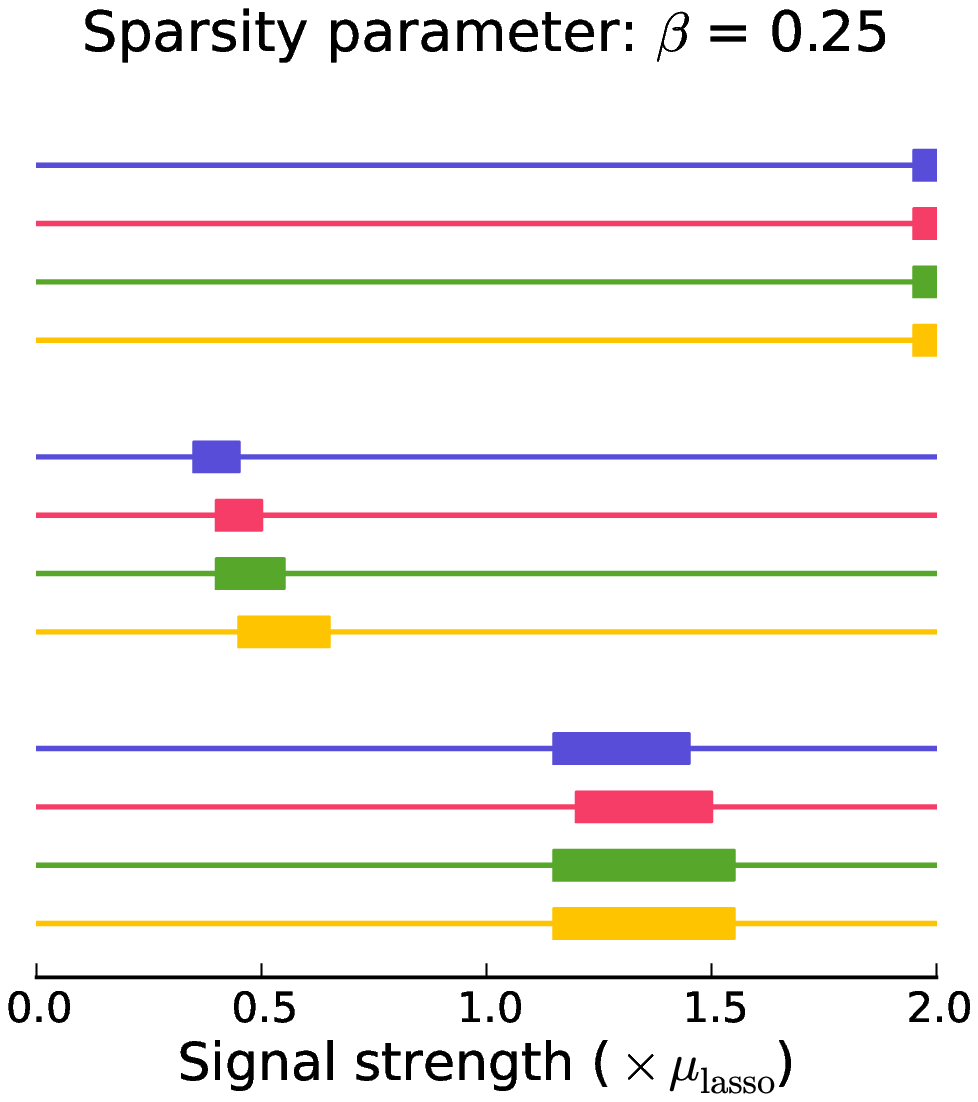}
    } 
\subfigure{
    \includegraphics[width=0.54\columnwidth]{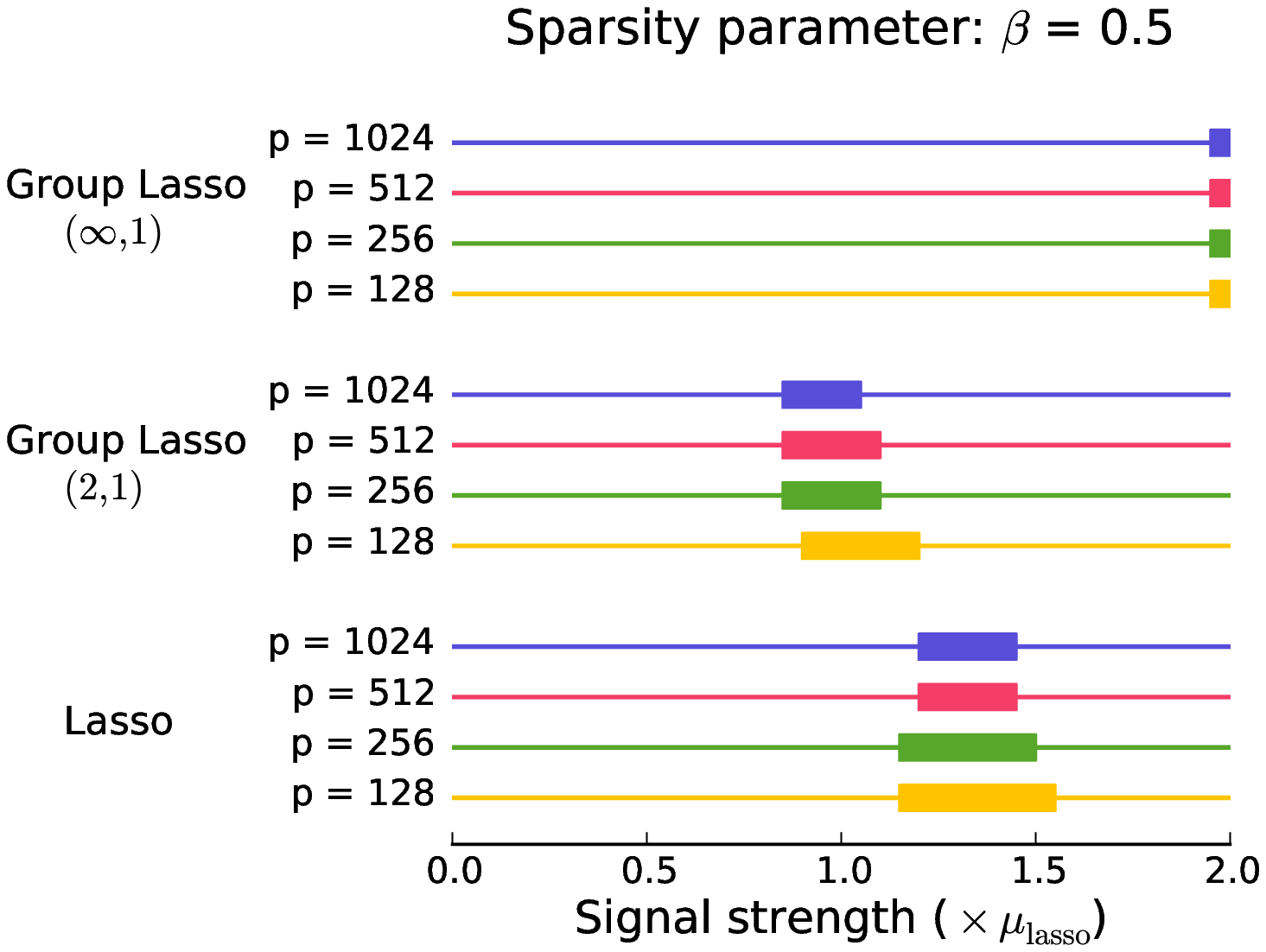}
    } 
\subfigure{
    \includegraphics[width=0.36\columnwidth]{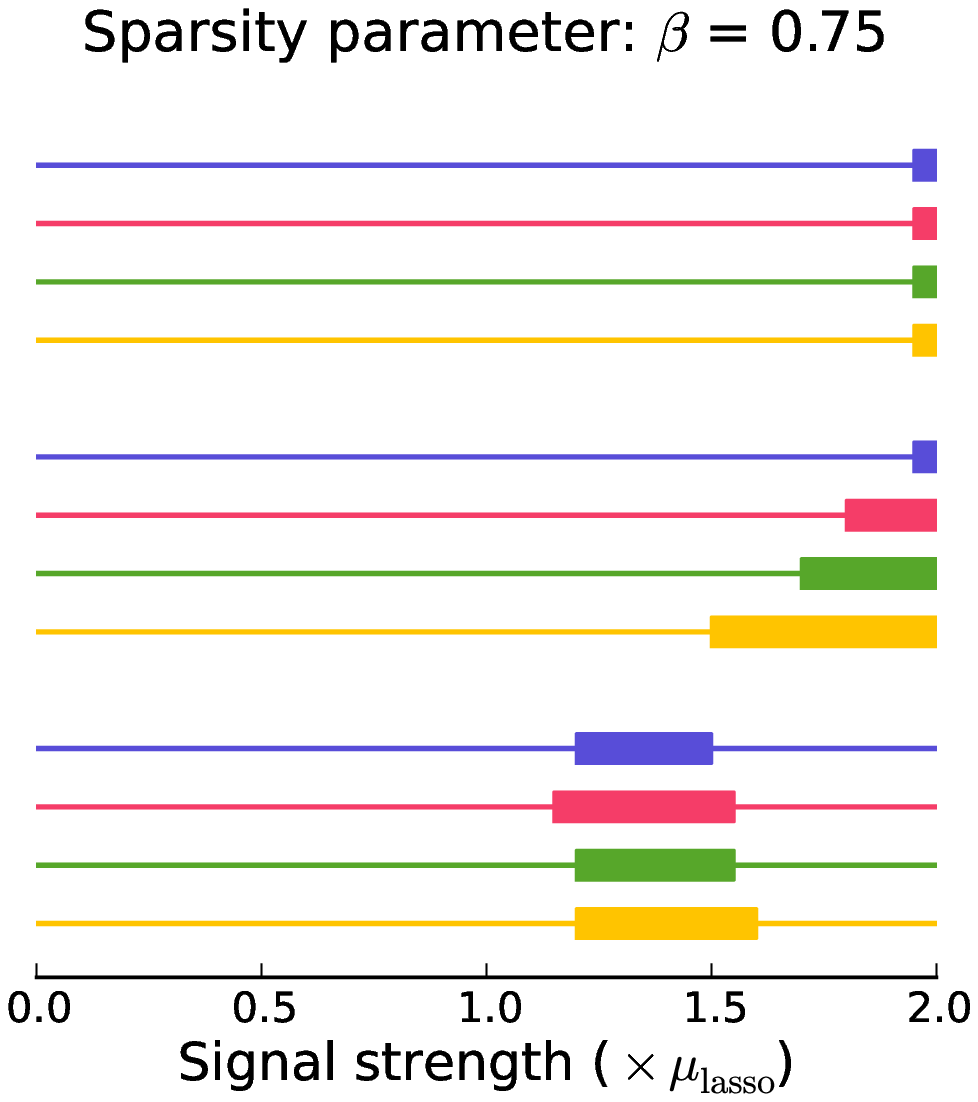}
    } 
    \caption{The probability of success for the Lasso for the problem
      of estimating $S$ plotted against the signal strength, which is
      varied as a multiple of $\mu_{\rm lasso}$ defined in
      \eqref{eq:exp:lasso}. A rectangle on each horizontal line
      represents points at which the probability $\PP[\hat S = S]$ is
      between $0.05$ and $0.95$. Different subplots represent the
      probability of success as the sparsity parameter $\beta$
      changes.}
    \label{fig:lasso:exper_4}
\end{figure}

\subsection{Group Lasso}

Next, we focus on the empirical performance of the group Lasso with
the mixed $(2,1)$ norm. Figure~\ref{fig:group-lasso:exper_4} plots the
probability of success as a function of the signal strength. Using
theorem \ref{thm:group-lasso}, we set
\begin{equation}
  \label{eq:exp:group-lasso}
    \mu_{\rm group} = \sigma \sqrt{2(\sqrt{5}+4)}
         \sqrt{\frac{k^{-1/2 + \beta}}{1-c}}
         \sqrt{\ln \frac{(2s -\delta')(p-s)}{\alpha'\delta'}}
\end{equation}
where $c$ is defined in theorem \ref{thm:group-lasso}. Next, we
generate data according to \eqref{eq:normal-means-model} with all
elements $\{\mu_{ij}\}$ set to $\mu = \rho \mu_{\rm group}$, where
$\rho \in [0.05,2]$. The penalty parameter $\lambda$ is given by
\eqref{eq:group-lasso:lambda}. Figure~\ref{fig:group-lasso:exper_4}
plots probability of success as a function of the parameter $\rho$,
which controls the signal strength. A rectangle on a horizontal line
represents points at which the probability $\PP[\hat S = S]$ is
between $0.05$ and $0.95$. From each subfigure in
Figure~\ref{fig:group-lasso:exper_4}, we can observe that the
probability of success for the group Lasso transitions from $0$ to $1$
for the same value of the parameter $\rho$ for different values of
$p$, which indicated that, except for constants, our theory correctly
characterizes the scaling of $\mu_{\min}$. We observe also that the
group Lasso outperforms the Lasso when each non-zero row is not too
sparse, that is, when there is a considerable overlap of features
between different tasks.

\begin{figure}[t]
  \centering
  Probability of successful support recovery: group Lasso\\
\subfigure{
    \includegraphics[width=0.54\columnwidth]{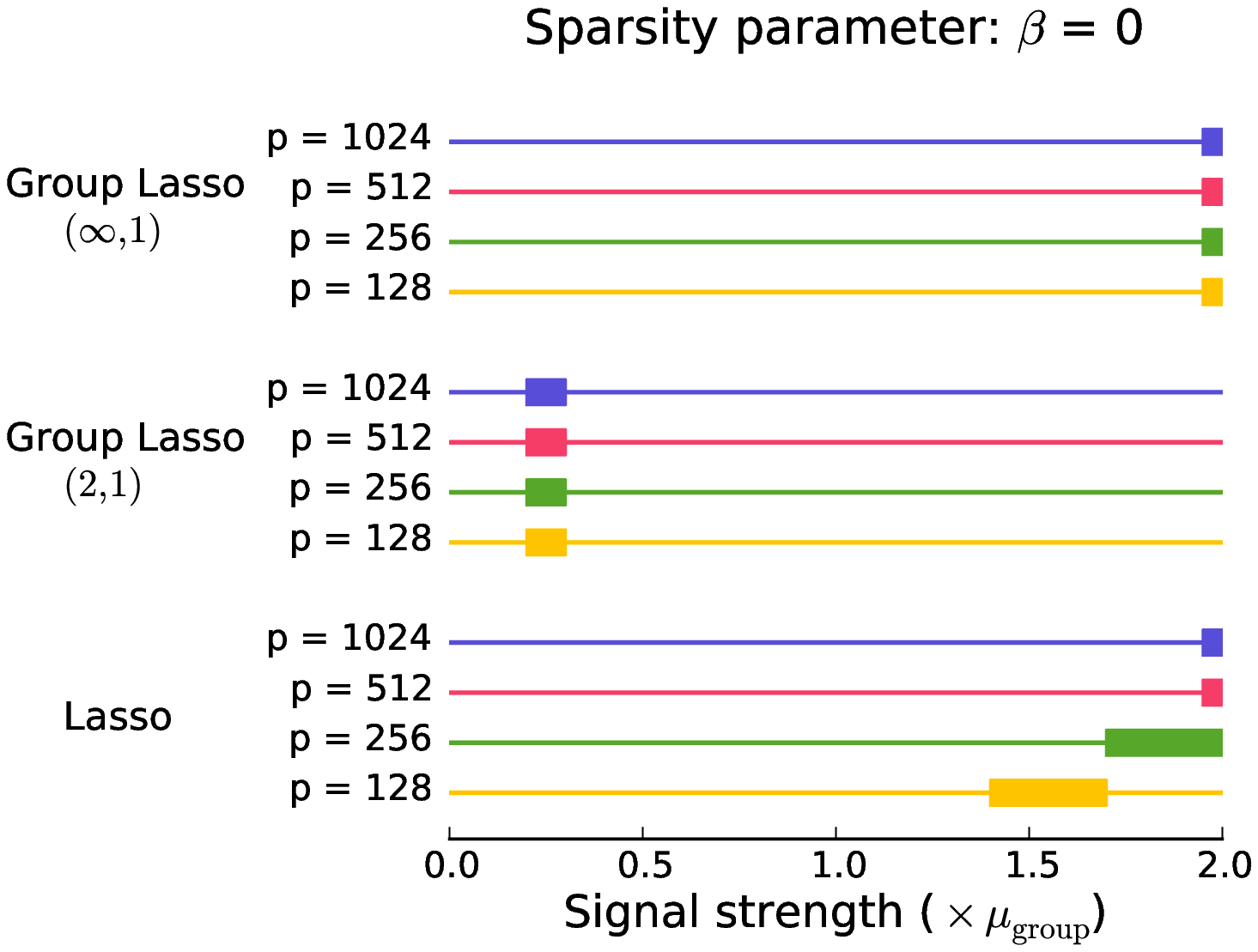}
    } 
\subfigure{
    \includegraphics[width=0.36\columnwidth]{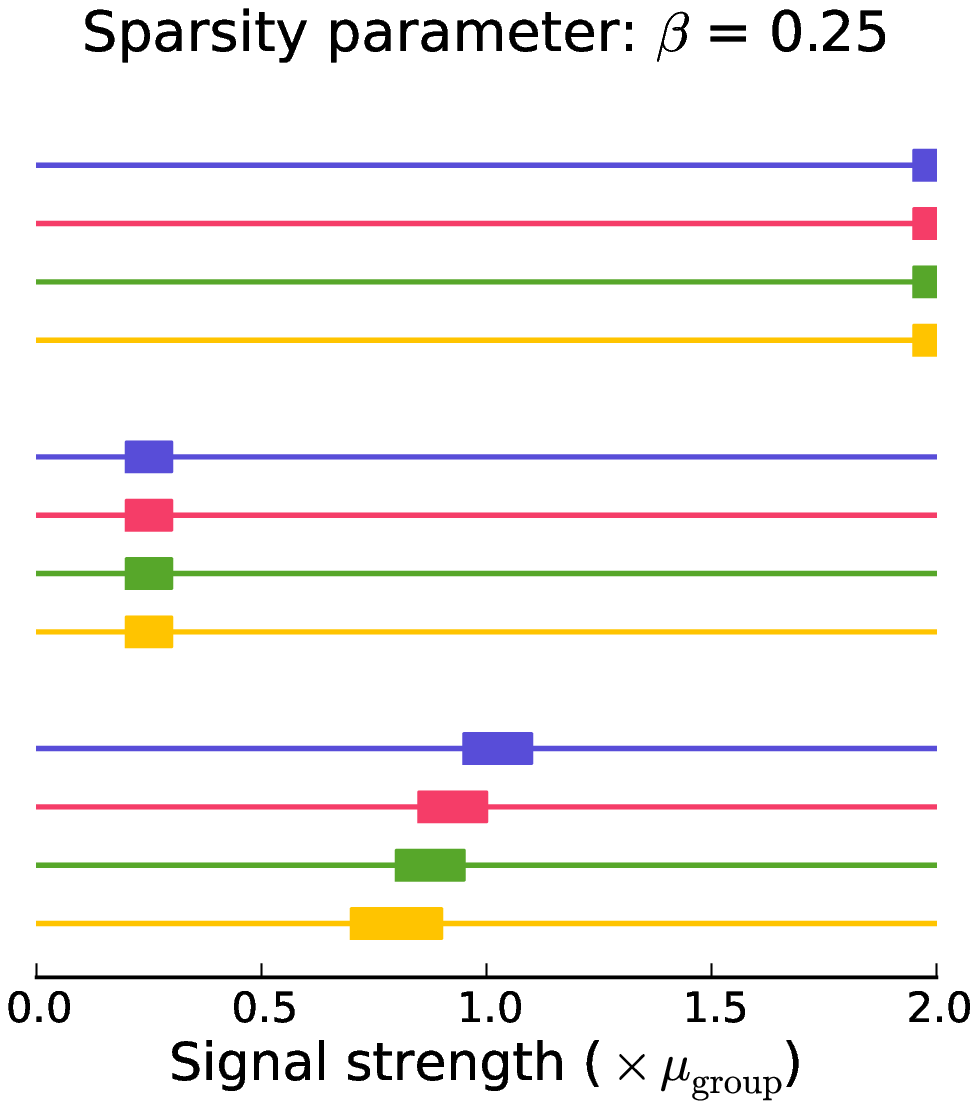}
    } 
\subfigure{
    \includegraphics[width=0.54\columnwidth]{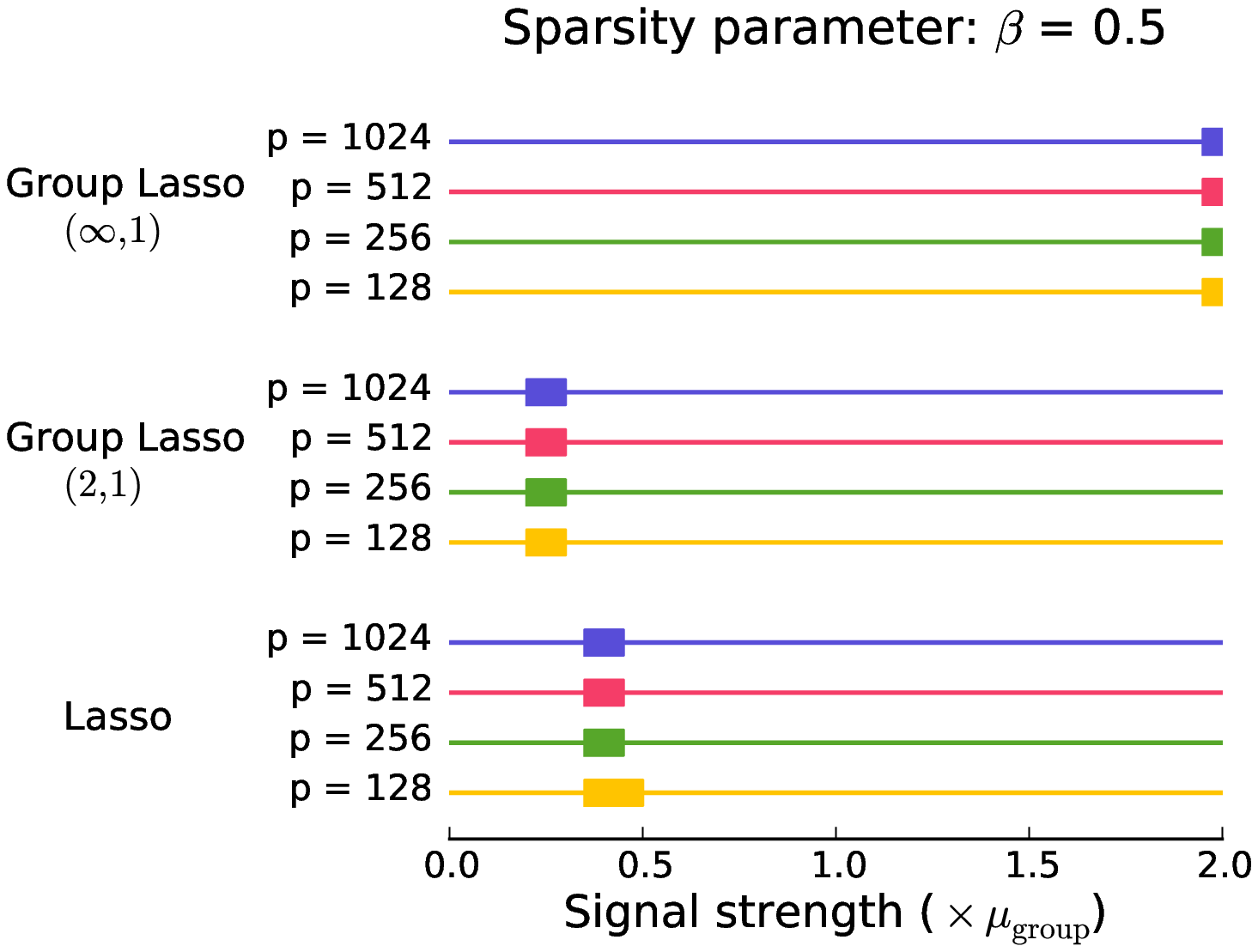}
    } 
\subfigure{
    \includegraphics[width=0.36\columnwidth]{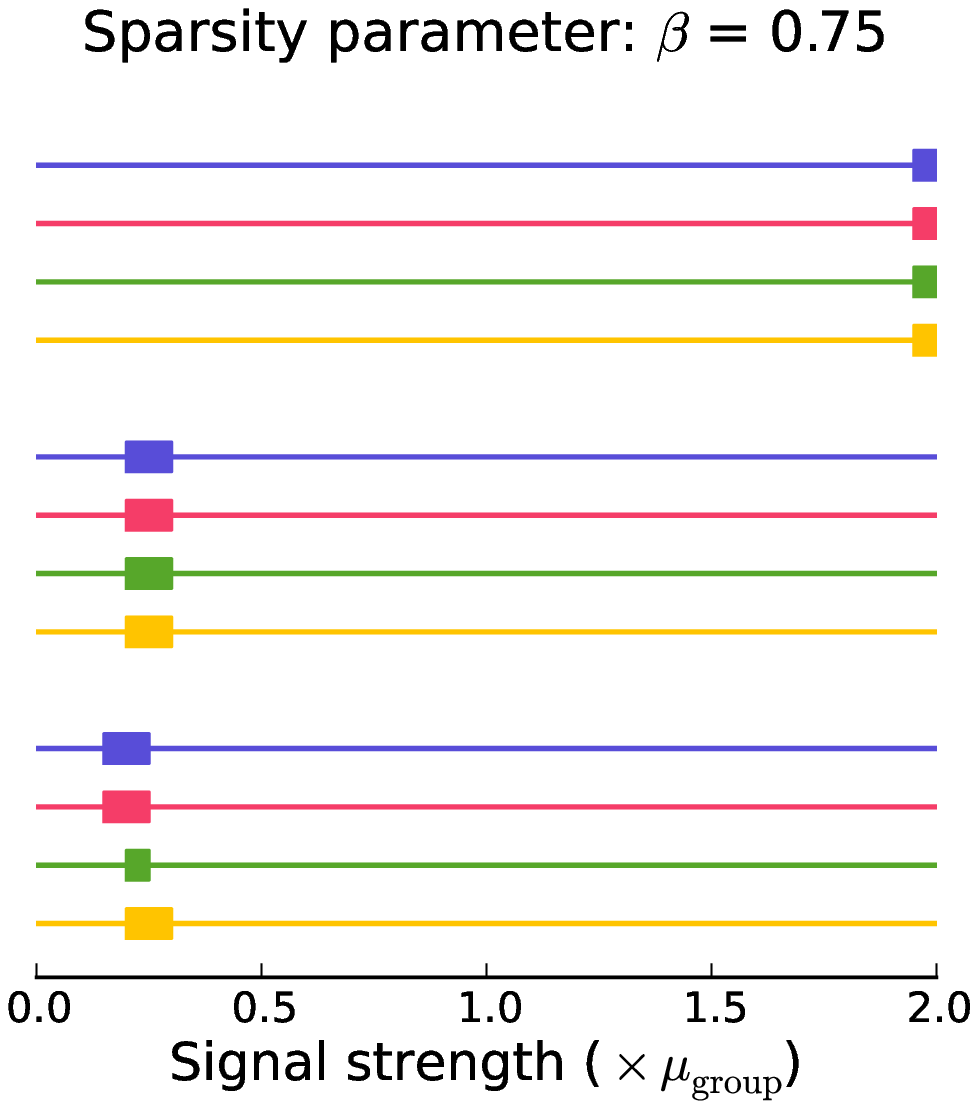}
    } 
    \caption{The probability of success for the group Lasso for the
      problem of estimating $S$ plotted against the signal strength,
      which is varied as a multiple of $\mu_{\rm group}$ defined in
      \eqref{eq:exp:group-lasso}. A rectangle on each horizontal line
      represents points at which the probability $\PP[\hat S = S]$ is
      between $0.05$ and $0.95$. Different subplots represent the
      probability of success as the sparsity parameter $\beta$
      changes.}
    \label{fig:group-lasso:exper_4}
\end{figure}

\subsection{Group Lasso with the mixed $(\infty, 1)$ norm}

Next, we focus on the empirical performance of the group Lasso with
the mixed $(\infty,1)$ norm. Figure~\ref{fig:group-lasso-inf:exper_4}
plots the probability of success as a function of the signal
strength. Using theorem \ref{thm:group-lasso-inf}, we set
\begin{equation}
  \label{eq:exp:group-lasso-inf}
    \mu_{\rm infty} = \frac{1+\tau}{1-c} k^{-1+\beta} \lambda
\end{equation}
where $\tau$ and $c$ are defined in theorem \ref{thm:group-lasso-inf}
and $\lambda$ is given by \eqref{eq:lambda-group-inf}. Next, we
generate data according to \eqref{eq:normal-means-model} with all
elements $\{\mu_{ij}\}$ set to $\mu = \rho \mu_{\rm infty}$, where
$\rho \in [0.05,2]$. Figure~\ref{fig:group-lasso-inf:exper_4} plots
probability of success as a function of the parameter $\rho$, which
controls the signal strength. A rectangle on a horizontal line
represents points at which the probability $\PP[\hat S = S]$ is
between $0.05$ and $0.95$. From each subfigure in
Figure~\ref{fig:group-lasso-inf:exper_4}, we can observe that the
probability of success for the group Lasso transitions from $0$ to $1$
for the same value of the parameter $\rho$ for different values of
$p$, which indicated that, except for constants, our theory correctly
characterizes the scaling of $\mu_{\min}$. We also observe that the
group Lasso with the mixed $(\infty, 1)$ norm never outperforms the
Lasso or the group Lasso with the mixed $(2,1)$ norm.

\begin{figure}[!h]
  \centering
  Probability of successful support recovery: group Lasso with
  the mixed $(\infty, 1)$ norm\\
\subfigure{
    \includegraphics[width=0.54\columnwidth]{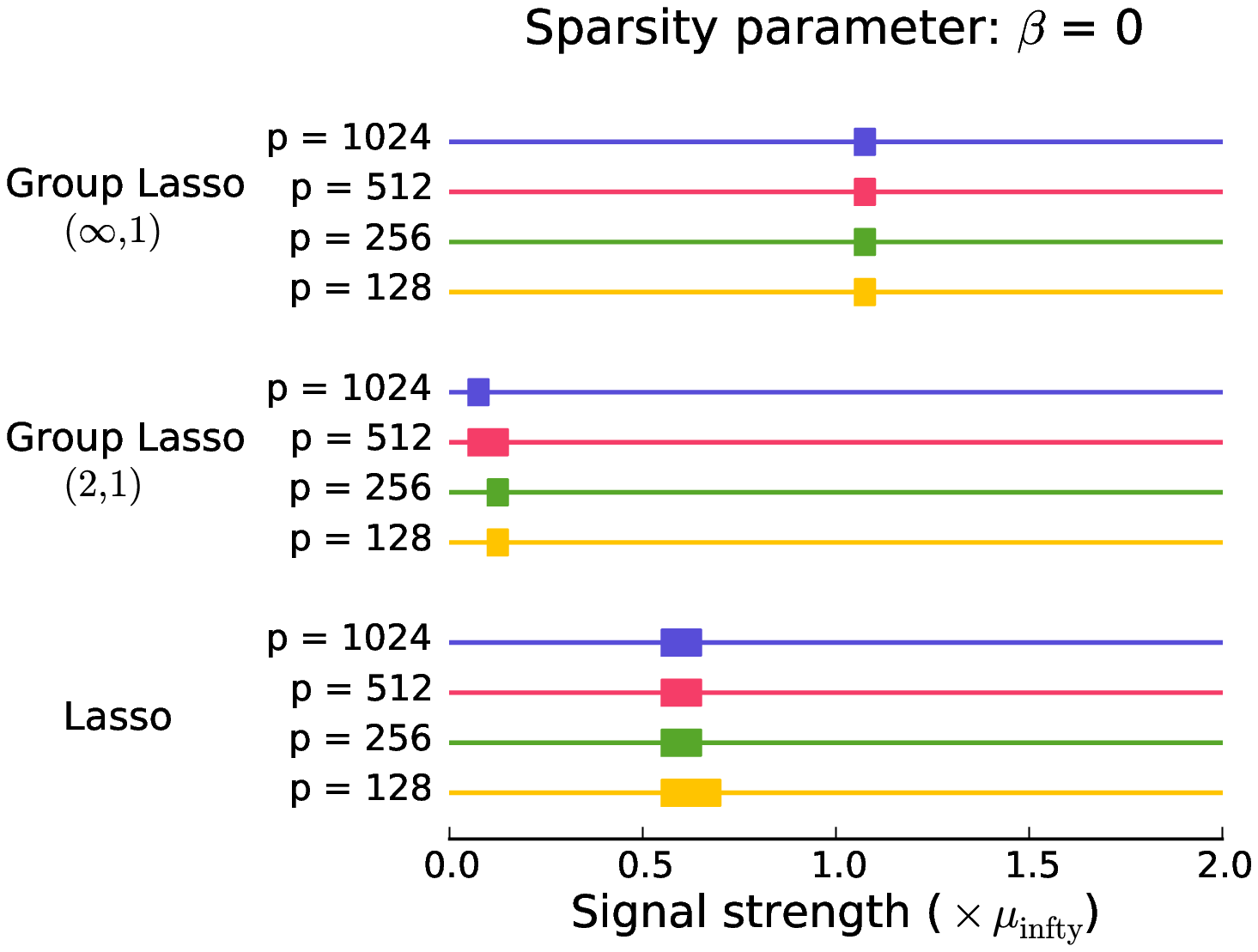}
    } 
\subfigure{
    \includegraphics[width=0.36\columnwidth]{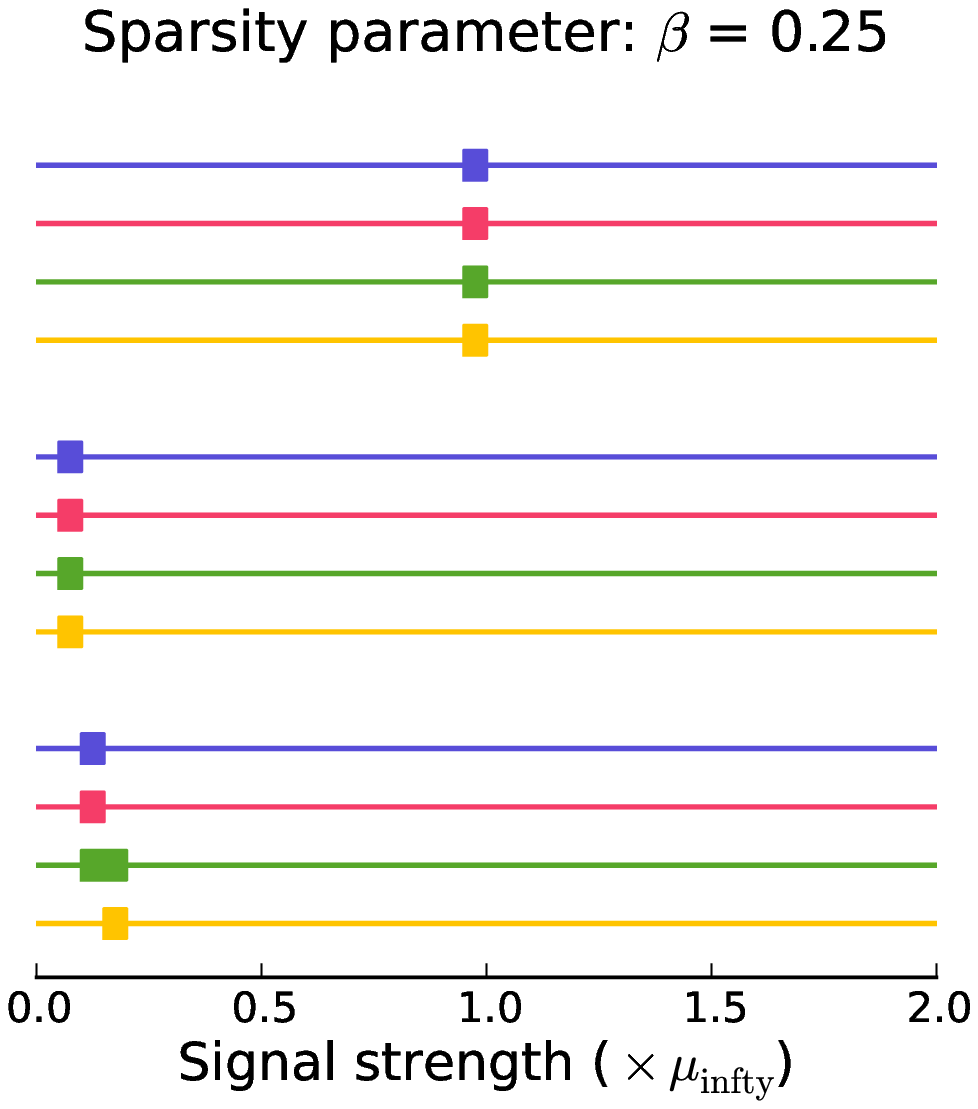}
    } 
\subfigure{
    \includegraphics[width=0.54\columnwidth]{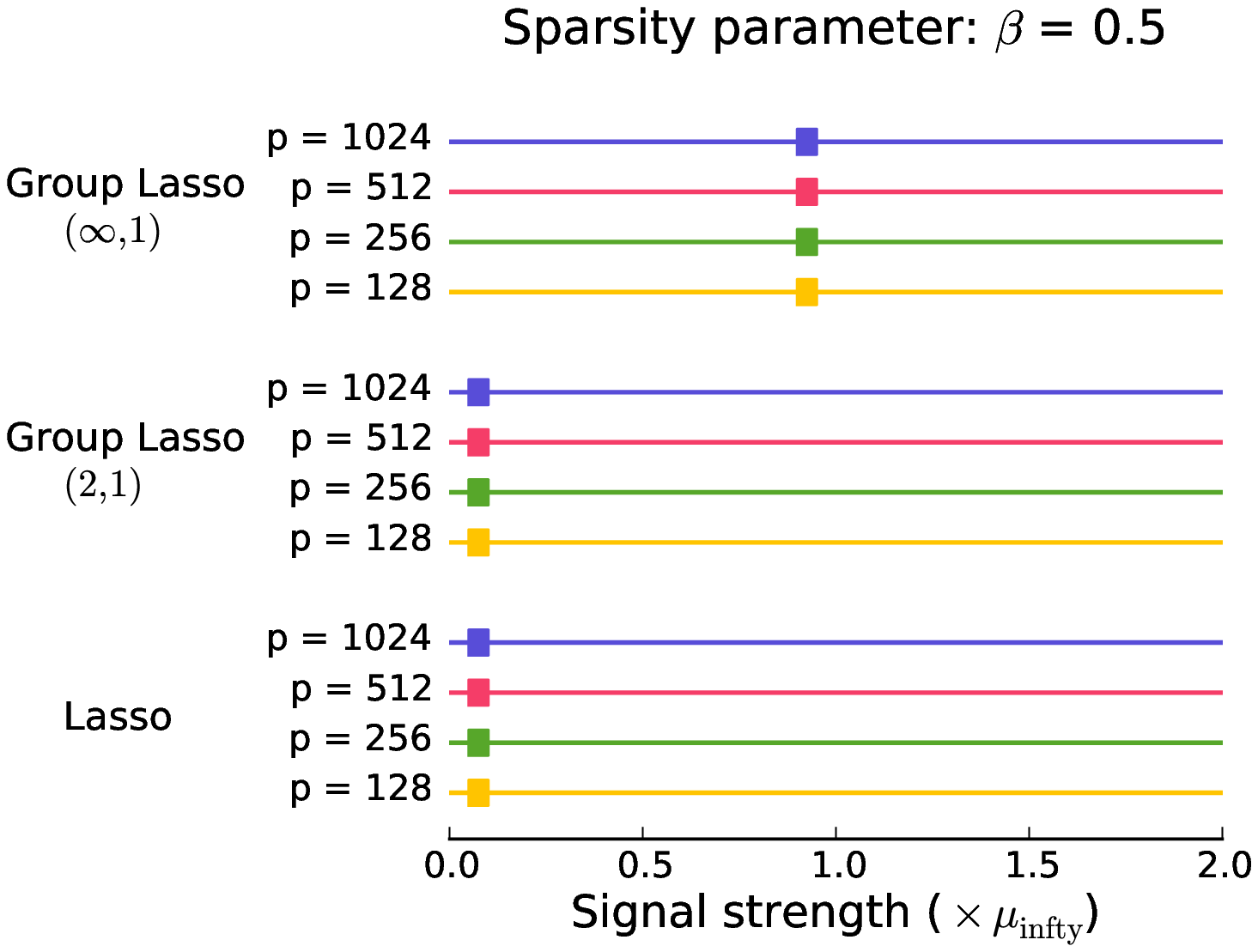}
    } 
\subfigure{
    \includegraphics[width=0.36\columnwidth]{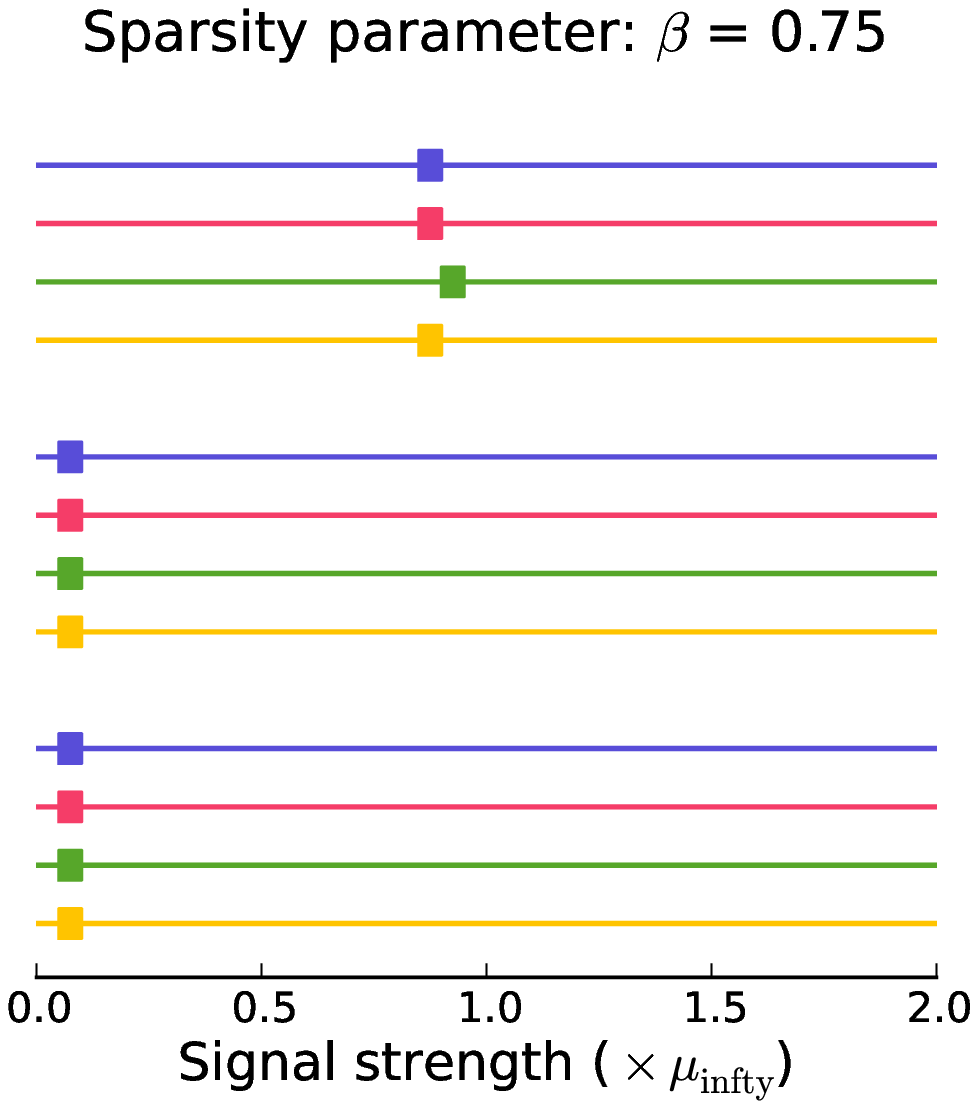}
    } 
    \caption{The probability of success for the group Lasso with mixed
      $(\infty, 1)$ norm for the problem of estimating $S$ plotted
      against the signal strength, which is varied as a multiple of
      $\mu_{\rm infty}$ defined in \eqref{eq:exp:group-lasso-inf}. A
      rectangle on each horizontal line represents points at which the
      probability $\PP[\hat S = S]$ is between $0.05$ and
      $0.95$. Different subplots represent the probability of success
      as the sparsity parameter $\beta$ changes.}
    \label{fig:group-lasso-inf:exper_4}
\end{figure}

\section{Discussion}
\label{sec:discussion}

We have studied the benefits of task sharing in sparse problems.
Under many scenarios, the group lasso outperforms the lasso.  The
$\ell_1/\ell_2$ penalty seems to be a much better choice for the group
lasso than the $\ell_1/\ell_\infty$.  However, as pointed out to us by
Han Liu, for screening, where false discoveries are less important
than accurate recovery, it is possible that the $\ell_1/\ell_\infty$
penalty could be useful.

We focused on the Normal means model.  While this model is obviously a
simplified model, it is extremely useful for theoretical study.  The
Normal means model is commonly used in Statistics and we hope that
this paper encourages researchers in machine learning to consider
wider use of this model as well.

\section{Proofs}

This section collects technical proofs of the results presented in the
paper. Throughout the section we use $c_1, c_2, \ldots$ to denote
positive constants whose value may change from line to line.

\subsection{Proof of Theorem~\ref{thm:lower-bound}}

Without loss of generality, we may assume $\sigma=1$. Let $\phi(u)$ be
the density of $\Ncal(0,1)$ and define $\PP_{\zero}$ and $\PP_{\one}$
to be two probability measures on $\RR^k$ with the densities with
respect to the Lebesgue measure given as
\begin{equation}
  \label{eq:density:row:0}
  f_0(a_1, \ldots, a_k) = \prod_{j \in [k]} \phi(a_j)
\end{equation}
and 
\begin{equation}
  \label{eq:density:row:1}
  f_1(a_1, \ldots, a_k) = \EE_Z\EE_m \EE_{\xi} \prod_{j \in m}
  \phi(a_j - \xi_j \mu_{\min}) \prod_{j \not\in m} \phi(a_j)
\end{equation}
where $Z \sim {\rm Bin}(k, k^{-\beta})$, $m$ is a random variable
uniformly distributed over $\Mcal(Z, k)$ and $\{\xi_j\}_{j \in [k]}$
is a sequence of Rademacher random variables, independent of $Z$ and
$m$.  A Rademacher random variable takes values $\pm 1$ with
probability $\frac{1}{2}$.

To simplify the discussion, suppose that $p-s+1$ is divisible by 2.
Let $T = (p-s+1)/2$. Using $\PP_{\zero}$ and $\PP_{\one}$, we
construct the following three measures,
\begin{equation*}
  \tilde \QQ =  \PP_{\one}^{s-1} \otimes \PP_{\zero}^{p-s+1},
\end{equation*}
\begin{equation*}
  \QQ_0 = \frac{1}{T} 
    \sum_{{j \in \{s, \ldots, p\}} \atop {j \text{ odd}}} 
     \PP_{\one}^{s-1}  \otimes \PP_{\zero}^{j-s} 
     \otimes \PP_{\one} \otimes \PP_{\zero}^{p-j}
\end{equation*}
and
\begin{equation*}
  \QQ_1 = \frac{1}{T} 
    \sum_{{j \in \{s, \ldots, p\}} \atop {j \text{ even}}} 
     \PP_{\one}^{s-1}  \otimes \PP_{\zero}^{j-s} 
     \otimes \PP_{\one} \otimes \PP_{\zero}^{p-j}.
\end{equation*}
It holds that
\begin{equation}
\label{eq:lower:framework}
\begin{aligned}
  \inf_{\hat \mu} \sup_{M \in \MM} 
   \PP_{M}[ S(M) \neq S(\hat \mu)] 
   &\geq \inf_{\Psi}
   \max\ \Big( \QQ_0(\Psi = 1), \QQ_1(\Psi = 0) \Big) \\
   & \geq \frac{1}{2} - \frac{1}{2} \norm{\QQ_0 - \QQ_1}_1,
\end{aligned}
\end{equation}
where the infimum is taken over all tests $\Psi$ taking values in $\{0, 1\}$
and $\norm{\cdot}_1$ is the total variation distance between
probability measures. For a readable introduction on lower bounds on
the minimax probability of error, see Section 2 in
\cite{tsybakov09introduction}. In particular, our approach is related
to the one described in Section 2.7.4. We proceed by upper bounding
the total variation distance between $\QQ_0$ and $\QQ_1$. 
Let $g = d\PP_{\one} / d\PP_{\zero}$ and let $u_i \in \RR^k$ for
each $i \in [p]$, then
\begin{equation*}
\begin{aligned}
  \frac{d\QQ_0}{d\tilde\QQ}&(u_1, \ldots, u_p) \\ &= 
  \frac{1}{T} 
      \sum_{{j \in \{s, \ldots, p\}} \atop {j \text{ even}}} 
      \prod_{i \in \{1, \ldots, s-1\} }
         \frac{d\PP_{\one}}{d\PP_{\one}}(u_i)
      \prod_{i \in \{s, \ldots, j-1\}}
         \frac{d\PP_{\zero}}{d\PP_{\zero}}(u_i)
      \frac{d\PP_{\one}}{d\PP_{\zero}}(u_j)      
      \prod_{i \in \{j+1, \ldots, p\}}
         \frac{d\PP_{\zero}}{d\PP_{\zero}}(u_i) \\ &=
  \frac{1}{T} 
      \sum_{{j \in \{s, \ldots, p\}} \atop {j \text{ even}}}
      g(u_j)
\end{aligned}
\end{equation*}
and, similarly, we can compute $d\QQ_1/d\tilde\QQ$. The following holds
\begin{equation}
\label{eq:tv-distance}
\begin{aligned}
\|&\QQ_0 - \QQ_1\|_1^2 \\
&= \Bigg(  \int\Big|\frac{1}{T} \big(
   \sum_{{j \in \{s, \ldots, p\}} \atop {j \text{ even}}} 
      g(u_j) - 
   \sum_{{j \in \{s, \ldots, p\}} \atop {j \text{ odd}}} 
      g(u_j) \big) \Big|
   \prod_{i \in \{s, \ldots, p\}} d\PP_{\zero}(u_i)
\Bigg)^2 \\
& \leq \frac{1}{T^2} \int \Big(
   \sum_{{j \in \{s, \ldots, p\}} \atop {j \text{ even}}} 
      g(u_j) - 
   \sum_{{j \in \{s, \ldots, p\}} \atop {j \text{ odd}}} 
      g(u_j)   
\Big)^2
   \prod_{i \in \{s, \ldots, p\}} d\PP_{\zero}(u_i) \\
& = \frac{2}{T}\big(\PP_{\zero}(g^2) - 1\big),
\end{aligned}
\end{equation}
where the last equality follows by observing that
\begin{equation*}
  \int \sum_{{j \in \{s, \ldots, p\}} \atop {j \text{ even}}} 
       \sum_{{j' \in \{s, \ldots, p\}} \atop {j' \text{ even}}} 
       g(u_j)g(u_{j'})
   \prod_{{i \in \{s, \ldots, p\}} \atop {i \text{ even}}} 
     d\PP_{\zero}(u_i) = T\ \PP_{\zero}(g^2) + T^2 - T
\end{equation*}
and 
\begin{equation*}
  \int \sum_{{j \in \{s, \ldots, p\}} \atop {j \text{ even}}} 
       \sum_{{j' \in \{s, \ldots, p\}} \atop {j' \text{ odd}}} 
       g(u_j)g(u_{j'})
   \prod_{{i \in \{s, \ldots, p\}}}
     d\PP_{\zero}(u_i) = T^2.
\end{equation*}
Next, we proceed to upper bound $\PP_{\zero}(g^2)$, using some ideas
presented in the proof of Theorem~1 in
\cite{baraud02non-asymptotic}. Recall definitions of $f_0$ and $f_1$
in \eqref{eq:density:row:0} and \eqref{eq:density:row:1}
respectively. Then $g = d\PP_{\one}/d\PP_{\zero} = f_1/f_0$ and we
have
\begin{equation*}
\begin{aligned}
  g(a_1, \ldots, a_k) &=
  \EE_{Z}\EE_m \EE_\xi \bigg[ \exp\Big(-\frac{Z\mu_{\min}^2}{2} +
  \mu_{\min} \sum_{j \in m} \xi_{j}a_{j}\Big)
  \bigg] \\
  & = \EE_Z\bigg[\exp\Big(-\frac{Z\mu_{\min}^2}{2}\Big)\ \EE_m \Big[
     \prod_{j \in m} \cosh(\mu_{\min}a_{j}) \Big]
  \bigg].
\end{aligned}
\end{equation*}
Furthermore, let $Z' \sim {\rm Bin}(k, k^{-\beta})$ be independent of
$Z$ and $m'$ uniformly distributed over $\Mcal(Z', k)$. The following
holds
\begin{equation*}
\begin{aligned}
  &\PP_{\zero}(g^2)\\
  & = \PP_{\zero} \bigg( \EE_{Z',Z} \Big[
 \exp\Big(-\frac{(Z+Z')\mu_{\min}^2}{2}\Big)\ 
 \EE_{m,m'} 
 \prod_{j \in m} \cosh(\mu_{\min}a_{j})
   \prod_{j \in m'} \cosh(\mu_{\min}a_{j}) 
  \Big] \bigg) \\
  & = \EE_{Z',Z} \Big[
 \exp\Big(-\frac{(Z+Z')\mu_{\min}^2}{2}\Big)
 \\ & \qquad\qquad\qquad \EE_{m,m'} \Big[
 \prod_{j \in m \cap m'} \int \cosh^2(\mu_{\min}a_j) \phi(a_j) da_j \\
 & \qquad\qquad\qquad\qquad
 \prod_{j \in m \triangle m'} \int \cosh(\mu_{\min}a_j) \phi(a_j) da_j
 \Big]
\Big],
\end{aligned}
\end{equation*}
where we use $m \triangle m'$ to denote $(m \cup m')\bks(m \cap m')$.
By direct calculation, we have that 
\begin{equation*}
  \int \cosh^2(\mu_{\min}a_j) \phi(a_j) da_j = \exp(\mu_{\min}^2)
  \cosh(\mu_{\min}^2)
\end{equation*}
and
\begin{equation*}
  \int \cosh(\mu_{\min}a_j) \phi(a_j) da_j = \exp(\mu_{\min}^2/2).
\end{equation*}
Since $\frac{1}{2}|m \triangle m'| + |m \cap m'| = (Z + Z')/2$, we
have that 
\begin{equation*}
\begin{aligned}
  \PP_{\zero}(g^2)
  & = \EE_{Z, Z'} \Big[E_{m, m'}\Big[\big(\cosh(\mu_{\min}^2)\big)^{|m
    \cap m'|}\Big] \Big]\\
  & = \EE_{Z, Z'}\Big[ \sum_{j = 0}^k
      p_j \big(\cosh(\mu_{\min}^2)\big)^j \Big] \\
  & = \EE_{Z, Z'}\Big[ \EE_X \Big[ 
        \cosh(\mu_{\min}^2)^{X} \Big] \Big],
\end{aligned}
\end{equation*}
where 
\begin{equation*}
  p_j = \left\{ 
    \begin{array}{cc}
      0 & \text{if } j < Z + Z' - k \text{ or } j > \min(Z, Z') \\
      \frac{ {Z' \choose j} { {k - Z'} \choose {Z-j} } }
         {{k \choose  Z}}  & \text{otherwise}
    \end{array}
  \right.
\end{equation*} 
and $P[X = j] = p_j$. Therefore, $X$ follows a hypergeometric
distribution with parameters $k$, $Z$, $Z'/k$. [The first parameter
denotes the total number of stones in an urn, the second parameter
denotes the number of stones we are going to sample without
replacement from the urn and the last parameter denotes the fraction
of white stones in the urn.]  Then following \cite[][p. 173; see also
\cite{baraud02non-asymptotic}]{aldous85exchangeability}, we know that
$X$ has the same distribution as the random variable $\EE[\tilde X |
\Tcal]$ where $\tilde X$ is a binomial random variable with parameters
$Z$ and $Z' / k$, and $\Tcal$ is a suitable $\sigma$-algebra. By
convexity, it follows that
\begin{equation*}
\begin{aligned}
  \PP_{\zero}(g^2)
  & \leq \EE_{Z, Z'}\Big[ \EE_{\tilde X} \Big[ 
        \cosh(\mu_{\min}^2)^{\tilde X} \Big] \Big] \\
  & = \EE_{Z, Z'} \left[ 
    \exp\bigg( Z\ln
       \Big(1 + \frac{Z'}{k}\big(\cosh(\mu_{\min}^2) - 1\big)\Big)
       \bigg)\right]\\
  & = \EE_{Z'} \EE_{Z} \left[ 
    \exp\bigg( Z\ln
       \Big(1 + \frac{Z'}{k}u\Big)
       \bigg)\right]
\end{aligned}
\end{equation*}
where $\mu_{\min}^2 = \ln (1 + u + \sqrt{2u + u^2})$ with
\begin{equation*}
  u = \frac{\ln \Big( 1 + \frac{\alpha^2T}{2}\Big)}{2k^{1-2\beta}}.
\end{equation*}
Continuing with our calculations, we have that
\begin{equation}
  \label{eq:lower:bound:1}
\begin{aligned}
  \PP_{\zero}(g^2)
  & = \EE_{Z'} \exp\Big(
    k \ln\big( 1 + k^{-(1+\beta)}u Z'\big)
    \Big) \\
  & \leq \EE_{Z'}  \exp\Big(
    k^{-\beta}uZ'
    \Big) \\
  & = \exp\bigg(
     k \ln \Big(1 + k^{-\beta}
     \big(\exp(k^{-\beta}u\big) - 1)\Big)
  \bigg) \\
  & \leq \exp\Big(k^{1-\beta}
     \big(\exp\big(k^{-\beta}u\big) - 1\big)\Big) \\
  & \leq \exp \Big(2k^{1-2\beta}u\Big) \\
  & = 1 + \frac{\alpha^2T}{2},
\end{aligned}
\end{equation}
where the last inequality follows 
since $k^{-\beta}u < 1$
for all large $p$.
Combining \eqref{eq:lower:bound:1} with
\eqref{eq:tv-distance}, we have that
\begin{equation*}
\|\QQ_0 - \QQ_1\|_1 \leq \alpha,
\end{equation*}
which implies that
\begin{equation*}
  \inf_{\hat \mu} \sup_{M \in \MM} 
   \PP_{M}[ S(M) \neq S(\hat \mu)] 
   \geq \frac{1}{2} - \frac{1}{2} \alpha.
\end{equation*}

\subsection{Proof of Theorem~\ref{thm:lasso}}
\label{sec:proof-theor-lasso}

  Without loss of generality, we can assume that $\sigma = 1$ and
  rescale the final result. For $\lambda$ given in
  \eqref{eq:lower-bound-lambda-lasso}, it holds that $\PP[|\Ncal(0,1)
  \geq \lambda] = o(1)$. For the probability defined in
  \eqref{eq:binomial-parameter}, we have the following lower bound
  \begin{equation*}
    \pi_k = (1-\epsilon)\PP[|\Ncal(0, 1)| \geq \lambda] + 
            \epsilon \PP[|\Ncal(\mu_{\min}, 1)| \geq \lambda]
          \geq \epsilon \PP[\Ncal(\mu_{\min}, 1) \geq \lambda].
  \end{equation*}
  We prove the two cases separately.\\
  {\bf Case 1:} {\it Large number of tasks.} By direct calculation
  \begin{equation*}
    \pi_k \geq
    \epsilon\PP[\Ncal(\mu_{\min}, 1) \geq \lambda] = \frac{1}{\sqrt{4
        \pi \log k}\big(\sqrt{1 + C_{k,p,s}} - \sqrt{r}\big)}
    k^{-\beta - \big(\sqrt{1 + C_{k,p,s}} - \sqrt{r}\big)^2 } =: \underline{\pi_k}.
  \end{equation*}
  Since $1 - \beta > \big(\sqrt{1 + C_{k,p,s}} - \sqrt{r}\big)^2$,
  using lemma~\ref{lem:binomial}, $\PP[{\rm Bin}(k, \pi_k) = 0]
  \rightarrow 0$ as $n \rightarrow \infty$. We can conclude that as
  soon as $k\underline{\pi_k} \geq \ln ( s / \delta' )$, it holds that
  $\PP[S(\hat \mu^{\ell_1}) \neq S] \leq \alpha$.  \\
  {\bf Case 2:} {\it Medium number of tasks.} When $\mu_{\min} \geq
  \lambda$, it holds that 
  \begin{equation*}
    \pi_k \geq
    \epsilon\PP[\Ncal(\mu_{\min}, 1) \geq \lambda] \geq
    \frac{k^{-\beta}}{2}.
  \end{equation*}
  Using lemma~\ref{lem:binomial}, we can conclude that as soon as
  $k^{1-\beta}/2 \geq \ln(s / \delta')$, it holds that $\PP[S(\hat
  \mu^{\ell_1}) \neq S] \leq \alpha$.

\subsection{Proof of Theorem~\ref{thm:group-lasso}}
\label{sec:proof-theor-group}

  Using lemma~\ref{lem:chernoff-binomial}, $\PP[{\rm Bin}(k,
  k^{-\beta}) \leq (1-c)k^{1-\beta}] \leq \delta' / 2s$ for $c=\sqrt{2
    \ln (2s / \delta') / k^{1-\beta}}$. For $i \in S$, we have that
  \begin{equation*}
    \PP[S_k(i) \leq \lambda] \leq \frac{\delta'}{2s} + 
    \Big(1-\frac{\delta'}{2s}\Big)
    \PP\bigg[S_k(i) \leq \lambda\ \big|\ 
    \Big\{\norm{\theta_i}_2^2 \geq (1-c)k^{1-\beta}\mu_{\min}^2\Big\}\bigg].
  \end{equation*}
  Therefore, using lemma~\ref{lem:baraud-chi-square-test} with $\delta
  = \delta'/(2s - \delta')$, if follows that $\PP[S_k(i) \leq \lambda]
  \leq \delta'/(2s)$ for all $i \in S$ when
  \begin{equation*}
    \mu_{\min} \geq \sigma \sqrt{2(\sqrt{5}+4)}
         \sqrt{\frac{k^{-1/2 + \beta}}{1-c}}
         \sqrt{\ln \frac{2e(2s -\delta')(p-s)}{\alpha'\delta'}}.
  \end{equation*}
  Since $\lambda = t_{n,\alpha'/(p-s)}\sigma^2$, $\PP[S_k(i) \geq
  \lambda] \leq \alpha'/(p-s)$ for all $i \in S^c$. We can conclude
  that $\PP[S(\hat \mub^{\ell_1/\ell_2}) \neq S] \leq \alpha$.

\subsection{Proof of Theorem~\ref{thm:group-lasso-inf}}
\label{sec:proof-theor-group-inf}

Without loss of generality, we can assume that $\sigma = 1$.
Proceeding as in the proof of theorem~\ref{thm:group-lasso}, $\PP[{\rm
  Bin}(k, k^{-\beta}) \leq (1-c)k^{1-\beta}] \leq \delta' / 2s$ for
$c=\sqrt{2 \ln (2s / \delta') / k^{1-\beta}}$ using
lemma~\ref{lem:chernoff-binomial}. Then for $i \in S$ it holds that
\begin{equation*}
  \PP[\sum_j |Y_{ij}| \leq \lambda] \leq \frac{\delta'}{2s} + 
  \Big(1-\frac{\delta'}{2s}\Big)
  \PP[(1-c)k^{1-\beta}\mu_{\min} + z_k \leq \lambda],
\end{equation*}
where $z_k \sim \Ncal(0, k)$. Since $(1-c)k^{1-\beta}\mu_{\min} \geq
(1 + \tau) \lambda$, the right-hand side of the
above display can upper bounded as 
\begin{equation*}
  \frac{\delta'}{2s} + \Big(1-\frac{\delta'}{2s}\Big)
  \PP[\Ncal(0, 1) \geq \tau\lambda/\sqrt{k}] \leq 
  \frac{\delta'}{2s} + \Big(1-\frac{\delta'}{2s}\Big)
  \frac{\delta'}{2s - \delta'} \leq \frac{\delta'}{s}.
\end{equation*}
The above display gives us the desired control of the type two error, and we
can conclude that $\PP[S(\hat \mub^{\ell_1/\ell_\infty}) \neq S] \leq
\alpha$.


\acks{We would like to thank Han Liu for useful discussions.}

\appendix
\section*{Appendix A.}
\label{app:known-results}

We provide in this section some known results that are used in the paper.

\begin{lemma} 
  \label{lem:tail-bound-normal}
  Let $X \sim \Ncal(0, 1)$, then $\mathbb{P}[|X| \geq \lambda] \leq
  \frac{2}{\sqrt{2\pi}\lambda}\exp(-\frac{\lambda^2}{2})$.
\end{lemma}
\begin{proof}
  Since $x/\lambda > 1$ for $x \in (\lambda, \infty)$, by direct
  calculation 
  \begin{equation*}
    \PP[X \geq \lambda] = 
     \frac{1}{\sqrt{2\pi}} \int_{\lambda}^\infty\exp \left(-\frac{x^2}{2}\right) dx    
    \leq \frac{1}{\sqrt{2\pi}} 
     \int_{\lambda}^\infty \frac{x}{\lambda}\exp\left(-\frac{x^2}{2}\right) dx
    = \frac{1}{\sqrt{2\pi}\lambda}\exp\left(-\frac{\lambda^2}{2}\right)
  \end{equation*}
  and $\PP[|X| \geq \lambda] \leq 2\PP[X \geq \lambda]$. 
\end{proof}

\begin{lemma} \label{lem:binomial}
If $z_k \sim \operatorname{Bin}(k, \pi_k)$, then for all $k \geq 1$ and
all $\pi_k \in (0,1)$ it holds that 
\begin{equation*}
  \PP[z_k = 0] \leq \exp(-k\pi_k).
\end{equation*}
\end{lemma}
\begin{proof}
$\mathbb{P}[z_k = 0] = (1 - \pi_k)^k = \exp(-k\log(\frac{1}{1 -
  \pi_k})) = \exp(-k(\pi_k + \Ocal(\pi_k^2))) \leq \exp(-k\pi_k).$
\end{proof}

\begin{lemma}{\it \citep{chernoff81noteinequality}}
  \label{lem:chernoff-binomial}
  If $z_k \sim {\rm Bin}(k, \pi_k)$, then 
  \begin{equation*}
    \PP[z_k \leq k\pi_k - t] \leq \exp(-t^2 / (2k\pi_k))    
  \end{equation*}
  and
  \begin{equation*}
    \PP[z_k \geq k\pi_k + t] \leq \exp(-t^2 / (2(k\pi_k + t/3))).
  \end{equation*}
\end{lemma}
\begin{proof}
  See \cite{chernoff81noteinequality}.
\end{proof}

\vskip 0.2in
\bibliography{biblio}

\end{document}